\documentclass{article}
\usepackage{arxiv}

\usepackage[utf8]{inputenc} 
\usepackage[T1]{fontenc}    
\usepackage{algorithmic}
\usepackage{algorithm}
\usepackage{array}
\usepackage{textcomp}
\usepackage{stfloats}           
\usepackage{url}
\usepackage{verbatim}
\usepackage{amsmath,amssymb,amsfonts,amsthm}
\usepackage{xfrac}
\usepackage{nicefrac}       
\usepackage{microtype}      
\usepackage{siunitx}
\usepackage{pifont}

\usepackage{graphicx}
\usepackage{booktabs}
\usepackage[table]{xcolor}
\usepackage[percent]{overpic}
\usepackage{multirow}
\usepackage{orcidlink}

\usepackage{orcidlink}
\usepackage[accsupp]{axessibility}

\newlength{\ColumnHeight}
\newtheorem{theorem}{Theorem}
\newtheorem{lemma}[theorem]{Lemma}
\newtheorem{proposition}[theorem]{Proposition}
\newtheorem{corollary}[theorem]{Corollary}
\theoremstyle{definition}

\usepackage{xr-hyper} 
\newcommand{\myexternaldocument}[1]{\externaldocument{#1}}
\newif\ifarxiv
\unless\ifarxiv \myexternaldocument{_supplement} \fi

\usepackage{cleveref}
\usepackage{subcaption}
\usepackage{natbib}
\usepackage{doi}

\title{Quasi-SVD: Learning a Lie-constrained matrix factorisation for real-time imaging}

\newif\ifuniqueAffiliation
\uniqueAffiliationtrue

\ifuniqueAffiliation 
\author{ \href{https://orcid.org/0000-0003-2786-9905}{\includegraphics[scale=0.06]{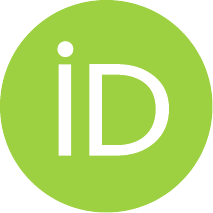}\hspace{1mm}Christopher Hahne}
    \\
	University of Bern\\
	3008 Bern, Switzerland \\
	\texttt{christopher.hahne@unibe.ch} \\
}
\else
\usepackage{authblk}

\newbox{\orcid}\sbox{\orcid}{\includegraphics[scale=0.06]{orcid.pdf}} 
\author[1]{%
	\href{https://orcid.org/0000-0003-2786-9905}{\usebox{\orcid}\hspace{1mm}Christopher Hahne\thanks{\texttt{christopher.hahne@unibe.ch}}}%
}
\affil[1]{Department of Computer Science, University of Bern, 3008 Bern, Switzerland}
\fi

\renewcommand{\shorttitle}{Quasi-SVD}

\hypersetup{
pdftitle={Quasi-SVD: Learning a Lie-constrained matrix factorisation for real-time imaging},
pdfsubject={q-bio.NC, q-bio.QM},
pdfauthor={Christopher Hahne},
pdfkeywords={Singular value decomposition, Lie algebra, GPU parallelization, Real-time imaging},
}

\newif\ifJournalOption
\JournalOptionfalse %

\begin{document}
\maketitle

\author{Christopher Hahne\textsuperscript{\large\orcidlink{0000-0003-2786-9905}}
\thanks{This work is funded in part by the Hasler Foundation, Bern, CH [Grant number 22027]. 
}
}

\maketitle

\begin{abstract}
Singular Value Decomposition (SVD) underlies matrix factorisation tasks across many fields, with imaging applications demanding cost-efficient processing. Yet SVD algorithms are inherently sequential, constraining GPU throughput and limiting online deployment in imaging pipelines. This study introduces a fully parallelized matrix factorization framework for GPUs by enforcing matrix orthogonality on left singular vectors via Lie-parametrised algebra and recovering the remaining components through soft constraints. This asymmetric constraint design enables an efficient parallel decomposition with an orthogonal left factor, achieves high reconstruction fidelity and substantially accelerates computation relative to the exact SVD, with high throughput exceeding standard video frame rates. Performance is evaluated on multiple imaging tasks spanning complementary computational regimes: (1) an MNIST denoising benchmark at an intermediate scale with known ground truth, (2) Mueller matrix polarimetry for neurosurgical tissue characterisation, requiring massive batch processing of small matrices, and (3) spatio-temporal background subtraction for ultrasound localisation microscopy, requiring high-dimensional matrix separation. Across regimes and instruments, the proposed framework demonstrates robust domain transfer at various matrix scales and 2.6--67$\times$ speed-ups over the exact GPU SVD, and, unlike randomised SVD, accelerates large batches of small full-rank matrices. By prioritising downstream reconstruction fidelity over exact spectral recovery, the proposed SVD framework enables GPU-friendly matrix factorisation.
\end{abstract}


\section{Introduction}
\label{sec:intro}
Today's imaging pipelines increasingly rely on GPU acceleration to achieve real-time processing and playback. Yet a standard matrix decomposition in scientific imaging, the singular value decomposition (SVD), remains difficult to parallelise. The SVD lies at the computational core of applications ranging from spatio-temporal filtering in ultrasound localisation microscopy~(ULM)~\cite{demene2015spatiotemporal,errico2015ultrafast,baranger2018adaptive,hahne:24:tmi} to \if MRI coil compression~\cite{otazo2015low,zbontar2018fastmri,johnson2021evaluation,cole2022learned} and\fi Mueller matrix polarimetry (MMP) for surgical guidance~\cite{lu1996interpretation,moriconi2024near,Hahne:25:OpEx}. The paradox is that the hardware has become massively parallel while the algorithms we rely on still employ fundamentally sequential and iterative procedures, leaving SVD as a bottleneck in pipelines that otherwise operate in real-time.
For example, the SVD can throttle MMP to approximately 4~FPS~\cite{moriconi2024near,Hahne:25:OpEx} or force ULM background suppression into an offline step~\cite{demene2015spatiotemporal,hahne:24:tmi}, well below the throughput required for live \if image-guided\fi decision-making.

This bottleneck is a result of the structural concept underlying the SVD. Classical SVD algorithms~\cite{golub1971singular,hansen1987truncated,hansen1990truncated,halko2011finding,musco2015randomized,STRUSKI:2024} rely on sequential transformations and iterative refinement that impose data dependencies between computational steps, preventing full use of the thousands of cores available on modern GPUs. %
The problem manifests at two computationally relevant extremes: (a) high-throughput batches of small matrix decompositions and (b) memory-intensive decompositions of high-dimensional data, as in temporal subspace filtering over hundreds of frames. Neither is well served by solvers designed around sequential bidiagonalisation, and the resulting latency is what currently keeps SVD-dependent imaging pipelines offline.

Existing GPU-parallel implementations improve throughput~\cite{lahabar2009singular,nvidia2025cusolver,demmel1992jacobi,novakovic2015hierarchically} but retain algorithmic dependencies that limit scalability, particularly across the range of matrix sizes relevant to imaging. Learning-based approaches offer a more GPU-friendly alternative~\cite{xu2022sv,ryu2024operator}, but typically enforce orthogonality through soft optimisation or avoid explicit factor recovery, leaving numerical validity insufficiently constrained. Conversely, Lie-theoretic parameterisations provide matrix orthogonality, but have not been applied to decomposition algebra until now.

This work closes that gap with a Lie-theoretic, fully differentiable matrix factorisation framework that resolves both limitations simultaneously. It enforces orthogonality on a single factor via the matrix exponential over a skew-symmetric Lie algebra, while recovering the remaining components through soft constraints. As the decomposition reduces to batched matrix products, one architecture covers matrix sizes from $3\times3$ to $280\times280$, and its speed-up over the exact SVD grows with batch size. The asymmetric design deliberately trades exact spectral recovery for reconstruction fidelity and throughput, a trade-off validated directly on clinically motivated imaging regimes and generic benchmarks.

This study makes the following contributions. \textbf{(i)~Lie-constrained factorisation}: a learned decomposition whose left factor is constrained to $\mathrm{SO}(m)$ through the matrix exponential of a predicted skew-symmetric generator, with the remaining factors recovered in closed form. To our knowledge, this is the first use of this parametrisation in a learned SVD, lowering the left-factor orthogonality error by up to two orders of magnitude relative to soft penalties. \textbf{(ii)~Batch-parallel GPU computation}: as the decomposition consists of batched matrix products only, it achieves 2.6--67$\times$ speed-ups over the exact GPU SVD and 4--35$\times$ over a GPU symmetric eigensolver at the highest fidelity among learned factorisations and, unlike randomised SVD, also accelerates large batches of small full-rank matrices. \textbf{(iii)~Evaluation across matrix regimes}: MNIST denoising ($280\times280$, known ground truth), Mueller matrix polarimetry (large batches of $3\times3$ matrices, transfer to a second instrument without retraining) and ULM clutter filtering ($250\times250$), compared on the same hardware with exact, eigensolver, randomised and closed-form $3\times3$ solvers.

\section{Related work}
\label{sec:related}

\paragraph{Traditional SVD Methods}
Singular Value Decomposition decomposes a matrix $\mathbf{A} \in \mathbb{R}^{m \times n}$ as $\mathbf{A} = \mathbf{U} \mathbf{\Sigma} \mathbf{V}^{\top}$, where $\mathbf{U}\in\mathrm{O}(m)$ and $\mathbf{V}\in\mathrm{O}(n)$ are orthogonal matrices and $\mathbf{\Sigma}$ contains ordered singular values~\cite{golub2013matrix}. The full SVD has cubic time complexity $\mathcal{O}(\min(m,n) \cdot mn)$ via Householder bidiagonalisation and QR iteration~\cite{golub1971singular}, which impose sequential dependencies for batched matrix factorisations on modern parallel hardware and thus limit the throughput~\cite{lahabar2009singular}. %
Approximation methods like Truncated SVD~\cite{hansen1987truncated,hansen1990truncated,STOLL20122795} and Randomised SVD~\cite{halko2011finding,musco2015randomized} reduce complexity when only a few leading components are needed, but rely on repeated QR orthonormalisation, which library implementations execute matrix by matrix, and offer no saving for small full-rank matrices, whose random sketch spans the full space. Rank–revealing QR still hinges on pivoted column selection heuristics that require dependent branching and synchronization~\cite{chan1987rankrevealing,gu1996efficient}. Nyström methods only postpone the problem by reducing the effective matrix size~\cite{drineas2005nystrom,musco2017nystrom}, but the subsequent SVD on the reduced core is still \emph{sequential} and remains the runtime limit. As alternatives, Streaming SVD updates decomposition dynamically as new data arrives~\cite{BRAND200620} and k-SVD employs an iterative framework to learn sparse representations by optimizing dictionary atoms~\cite{aharon2006k,scetbon2021deep}. Despite their sophistication, these methods do not resolve the core structural barrier, which is to map cleanly to highly parallel computation architectures.

\paragraph{Parallel SVD Implementations} %
To overcome sequential bottlenecks, CUDA-based implementations~\cite{lahabar2009singular} 
accelerate SVD for large matrices using partial bidiagonalisation followed by QR or 
divide-and-conquer routines, although Householder transformations remain inherently 
sequential and limit parallelism~\cite{golub1971singular,novakovic2015hierarchically}. 
Jacobi-based algorithms~\cite{demmel1992jacobi,golub2013matrix,novakovic2015hierarchically} 
offer a more parallelisable alternative via rotations over independent 
$2\times2$ pivot submatrices. Modern GPU libraries such as NVIDIA cuSOLVER implement Jacobi-type solvers (e.g., \texttt{gesvdj})~\cite{nvidia2025cusolver}, but these batched routines are practically restricted to small matrices ($m,n \le 32$), falling back to sequential bidiagonalisation-based SVD for larger matrices.

\paragraph{Learning-Based Approaches} %
Recent work has explored neural networks for SVD-related tasks. SV-Learn~\cite{xu2022sv} uses MLPs to predict singular values but suffers from numerical instability due to soft constraints and cannot recover singular vectors. Dictionary‑learning methods such as Deep K‑SVD~\cite{scetbon2021deep} produce sparse representations via iterative dictionary updates which rely on solvers with loop‑heavy control flow that resists the parallelism targeted here. %
Operator-SVD~\cite{ryu2024operator} learns spectral decompositions of linear operators through nested low-rank approximation, achieving theoretical elegance but at significant computational cost unsuitable for real-time imaging. 
Beyond imaging, SVD also underlies spectral clustering layers, though the high computational cost of the decomposition often necessitates its replacement with dual autoencoder architectures to maintain end-to-end efficiency~\cite{yang2019deep}. %
In the large language model (LLM) domain, SVD has attracted attention for activation whitening~\cite{yuan2023asvd}, low-rank weight approximation~\cite{qsvd25}, and truncation-aware model compression~\cite{yuan2023asvd,wang2025svdllm}, underscoring the broad demand for efficient factorisation. %
More broadly, end-to-end (E2E) methods learn a desired reconstruction mapping directly from data offering fast inference capabilities for images, as for example Learned SVD (L-SVD)~\cite{brune2019learned} or SVD-inspired filter networks~\cite{cole2022learned} included as learning-based baselines. In contrast, the proposed study retains an explicit, mathematically principled matrix-factorisation structure while approximating the decomposition to increase image throughput.

\paragraph{Lie Groups and Orthogonal Constraints} 
The matrix exponential and Cayley transform map Lie algebras onto orthogonal groups, providing hard matrix orthogonality constraints. These tools are established in visual odometry~\cite{Whelan-RSS-15,teed2021tangent,hayoz2023learning,matsuki2024gaussian} and in orthogonality-constrained layers~\cite{lezcano2019cheap,li2020efficient}, yet no prior work uses them to enforce orthogonality in a learned decomposition of streaming matrices. Learned SVD methods instead impose orthogonality only implicitly through unconstrained optimisation~\cite{ryu2024operator} or omit it entirely~\cite{xu2022sv}, accumulating numerical drift. %
No prior method therefore satisfies both requirements of parallelism and orthogonality guarantees simultaneously. %
\section{Proposed method}
\label{sec:method}

This work targets an efficient, learnable SVD approximation that preserves reconstruction and downstream accuracy while minimizing GPU wall‑clock latency and keeping training cost practical.
The Quasi‑SVD framework comprises a spectrum of module variants (illustrated in Fig.~\ref{fig:arch:ablation}): analytic closed‑form, unconstrained neural net (UNN), Lie‑informed parametrisations, and end‑to‑end prediction where each trades off accuracy, latency, and training effort. Full derivations, work–span analyses, and implementation details appear in the supplement. %

\paragraph{General notation} In close analogy to the SVD, a factorisation of matrix $\mathbf{A}\in\mathbb{R}^{m\times n}$ is given by:
\begin{align}
	\mathbf{A} = \mathbf{X} \mathbf{D} \mathbf{Y}^{\top}, \quad \text{where}
    \label{eq:svd_approx}
\end{align}
\begin{align}
	\mathbf{D} = \mathrm{pdiag}(\mathbf{d}) = \begin{bmatrix}\mathrm{diag}(\mathbf{d}) \\ \mathbf{0}_{(m-n)\times n}\end{bmatrix} \in \mathbb{R}^{m\times n},
    \label{eq:svd_approx_diagonal}
\end{align}
with $\mathbf{X}\in \mathrm{SO}(m)$, $\mathbf{Y}\in\mathbb{R}^{n\times n}$, and $\mathbf d\in\mathbb{R}^n$ collecting the diagonal entries, i.e., $\mathbf d=\mathrm{pdiag}^{-1}(\mathbf D)=[d_1,\ldots,d_n]^\top$, ordered in descending magnitude. Here, $\mathrm{pdiag}(\cdot):\mathbb{R}^n\to\mathbb{R}^{m\times n}$ zero-pads $\mathrm{diag}(\cdot)$ with $m-n$ rows below and $\mathrm{pdiag}^{-1}(\cdot)$ recovers the diagonal entries of the leading $n\times n$ block, while $m\ge n$ without loss of generality (if $m<n$, transpose $\mathbf{A}$ and swap the resulting factors back).

\begin{figure*}[!ht]
  \ifJournalOption
\else
    \newcommand{\subfigheight}{4.25cm}
  \fi
  \begin{subfigure}[t]{0.496\textwidth}
    \vspace{0pt} 
    \parbox[c][\ColumnHeight][c]{\linewidth}{%
      \centering
      \includegraphics[height=\subfigheight,keepaspectratio]{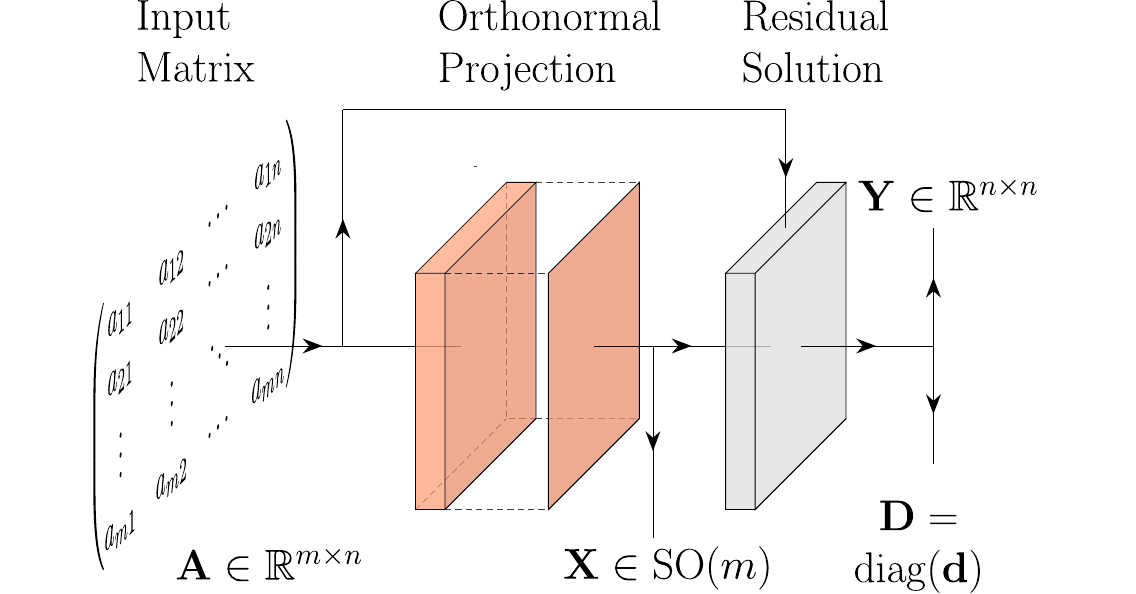}
    }
    \vspace{1.0em}
    \caption{Analytic model}\label{fig:arch:ablation:a}
  \end{subfigure}\hfill
  \begin{subfigure}[t]{0.496\textwidth}
    \vspace{0pt}
    \parbox[c][\ColumnHeight][c]{\linewidth}{%
      \centering
      \includegraphics[height=\subfigheight,keepaspectratio]{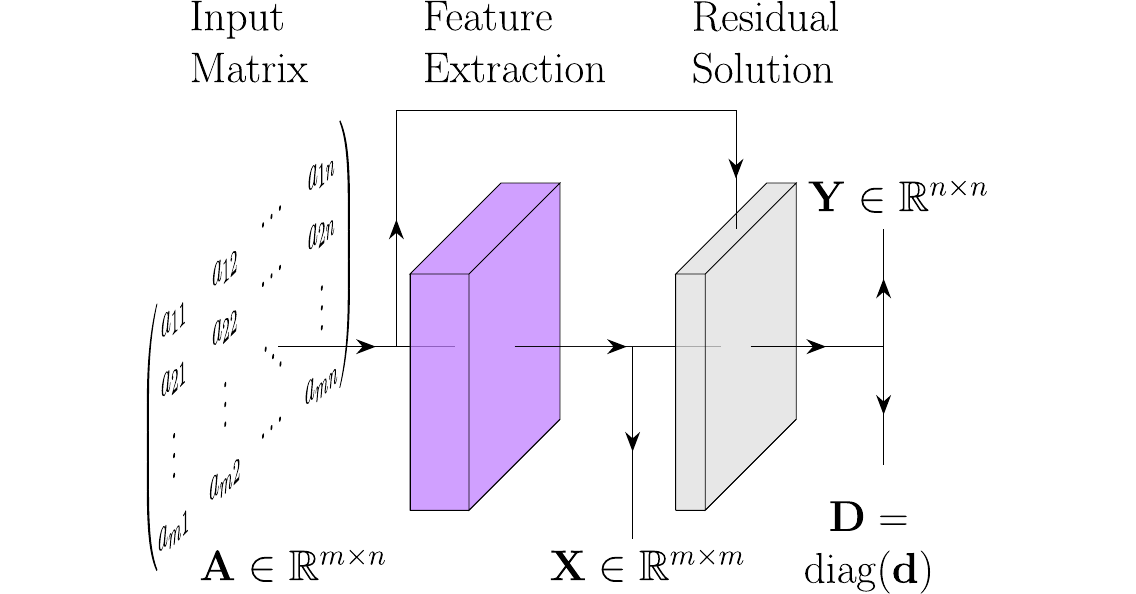}
    }
    \vspace{1.0em}
    \caption{Unconstrained neural net (UNN)}\label{fig:arch:ablation:b}
  \end{subfigure}\hfill
  \vspace{2.5em}
  \begin{subfigure}[t]{0.496\textwidth}
    \vspace{0pt}
    \parbox[c][\ColumnHeight][c]{\linewidth}{%
      \centering
      \includegraphics[height=\subfigheight,keepaspectratio]{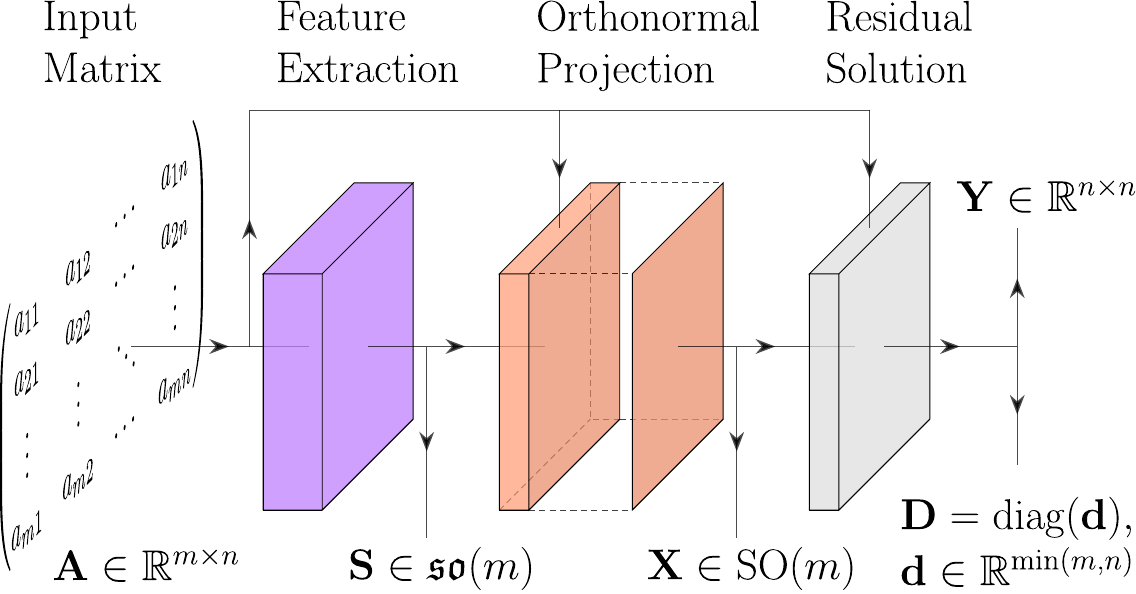}
    }
    \vspace{1.0em}
    \caption{Lie-based neural net (LieNN)}\label{fig:arch:ablation:c}
  \end{subfigure}\hfill
  \begin{subfigure}[t]{0.496\textwidth}
    \vspace{0pt}
    \parbox[c][\ColumnHeight][c]{\linewidth}{%
      \centering
      \includegraphics[height=\subfigheight,keepaspectratio]{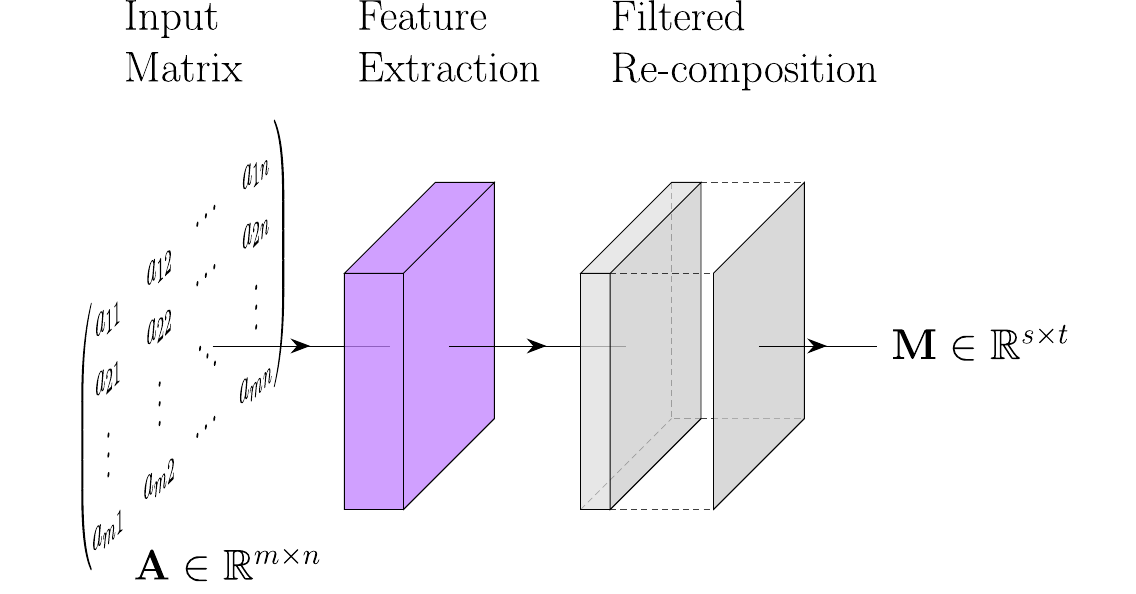}
    }
    \vspace{1.0em}
    \caption{End-to-end network (E2E)}\label{fig:arch:ablation:d}
  \end{subfigure}
    \caption{\textbf{Model variants.} The analytic model in \subref{fig:arch:ablation:a} provides a closed-form solution from the matrix exponential for orthonormal projection without learnable modules. The unconstrained neural network (UNN) in \subref{fig:arch:ablation:b} introduces trainable weights to factorise the input matrix. The Lie-based neural network (LieNN) in \subref{fig:arch:ablation:c} predicts skew-symmetric parameters mapped to an orthogonal factor via the matrix exponential, combining learned flexibility with hard orthogonality guarantees. The end-to-end (E2E) network in \subref{fig:arch:ablation:d} omits mathematical constraints entirely, producing only the composed matrix $\mathbf{M}$. While E2Es have been studied~\cite{brune2019learned,cole2022learned}, models in \subref{fig:arch:ablation:a} and \subref{fig:arch:ablation:c} contain the proposed Lie constraint. All model variants are used for an ablation analysis.}
    \label{fig:arch:ablation}
\end{figure*}

\paragraph{Analytic decomposition}
Orthogonality is central in many factorisation schemes underlying SVD. The present formulation enforces this structure directly inside the corresponding Lie group and parametrises orthogonal matrices through the Lie algebra $\mathfrak{so}(m)$. A skew-symmetric matrix $\mathbf{S}\in\mathfrak{so}(m)$ ensures that $\exp(\mathbf{S})\in\mathrm{SO}(m)$, making the matrix exponential an analytic mechanism for generating valid left/right factors. Given a matrix $\mathbf{A}$, the construction proceeds as follows: %
\begin{align}
    \mathbf{X} = \exp\!\left(\mathbf{S}\right) \, , \quad \text{where} \quad \mathbf{S}=\mathbf{A}-\mathbf{A}^\top \, ,
\end{align}
where $\mathbf{S}+\mathbf{S}^\top=\mathbf{0}$ enforces skew-symmetry. The exponential map $\exp(\cdot):\mathfrak{so}(m)\mapsto\mathrm{SO}(m)$ guarantees $\mathbf{X}$ is orthogonal. For rectangular inputs ($m\neq n$), $\mathbf{A}$ is zero-padded by $m-n$ columns to form a square matrix $\mathbf{A}_{\mathrm{pad}}\in\mathbb{R}^{m\times m}$ prior to constructing $\mathbf{S}$. 
\par
For higher-dimensional cases, the computational burden of the matrix exponential can be reduced by using the order-$p$ Taylor expansion $\exp_p(\cdot)$ as an approximation and, in the special case where $\mathbf{S}\in\mathfrak{so}(3)$, the Rodrigues' formula $\exp_3(\cdot)$ as a closed-form expression:
\begin{align}
    \exp_t\left(\mathbf{S},p\right) &= \sum_{r=0}^{p} \frac{\mathbf{S}^r}{r!}, \quad \text{while}
\end{align}
\begin{align}
    \exp_3\left(\mathbf{S}\right) \;\overset{\mathbf{S}\in\mathfrak{so}(3)}{=}\; \mathbf{I} +\frac{\sin\theta}{\theta}\mathbf{S} + \frac{1-\cos\theta}{\theta^2}\mathbf{S}^2,
\end{align}
where $\theta=\lVert\mathbf{S}\rVert_F/\sqrt{2}$ and $\lVert\cdot\rVert_F$ denotes the Frobenius norm. These yield a differentiable, algebraically grounded parameterisation of orthogonal matrix components with determinant $+1$. 
From $\mathbf{X}$, the right factor $\mathbf{Y}$ and the diagonal entries $\mathbf{d}$ are estimated via:
\begin{align}
    \mathbf{Y}=\mathbf{A}^\top\mathbf{X}_n\mathbf{D}^\dagger, \,\,\,\,\, \text{where} \,\,\,\,\,
    \mathbf{D}^\dagger=\mathrm{diag}(\mathbf{d}^{-1}), \quad \text{and}
\end{align}
\begin{align}
    \mathbf{d} = \sqrt{\left\lvert\,\mathrm{diag}^{-1}\!\left(\mathbf{X}^\top_n\mathbf{A}\mathbf{A}^\top\mathbf{X}_n\right)\right\rvert+\varepsilon}\,,
    \label{eq:diag_closed}
\end{align}
with $\mathbf{X}_n\in\mathbb{R}^{m\times n}$ denoting the first $n$ columns of $\mathbf{X}$. Here, $\varepsilon=10^{-9}$ avoids a gradient singularity in $\sqrt{\cdot}$ near zero and $\mathbf{d}$ is separately clamped at $10^{-11}$ before inversion to prevent division by zero.
%
\par
\paragraph{Unconstrained neural decomposition}
Extending on the work of SV-Learn~\cite{xu2022sv}, a viable design choice for full neural SVD prediction is to let a network $g_\theta$ predict singular vectors:
\begin{align}
    \mathbf{X}_g=g_\theta\left(\mathbf{A}\right) \quad \text{where} \quad g_\theta: \mathbb{R}^{m\times n}\mapsto\mathbb{R}^{m\times m}
\end{align}
with $\mathbf{X}_g\in\mathbb{R}^{m\times m}$. To satisfy $\mathbf{A}=\mathbf{X}_g\mathbf{D}_g\mathbf{Y}_g^\top$, one borrows the concepts from Eq.~(\ref{eq:diag_closed}) to compute $\mathbf{D}_g$ and $\mathbf{Y}_g$, respectively. This unconstrained model is illustrated in Fig.~\ref{fig:arch:ablation:b}.

\paragraph{Lie-informed neural decomposition}
To enforce matrix orthogonality in a learning context, a neural model $f_\theta$ predicts the upper triangle of the skew-symmetric matrix $\mathbf{s}$ via:
\begin{align}
    \mathbf{s} = f_\theta\left(\mathbf{A}\right) \quad \text{where} \quad f_\theta: \mathbb{R}^{m\times n}\mapsto\mathbb{R}^{L} \label{eq:neural}
\end{align}
where $\mathbf{s}=\left[s_1, s_2, \dots, s_L\right]^\top$ contains $L$ elements that are scattered into the upper triangle of the skew-symmetric matrix through the mapping $F:\mathbb{R}^{L}\mapsto\mathfrak{so}(m)$, with $F(\mathbf{s})_{ij}=s_{k}$ for $i<j$ (and $F(\mathbf{s})_{ji}=-s_{k}$), $k$ indexing the upper-triangular entries in row-major order.

From this, the Lie-based neural network (LieNN) predicts the matrix $\mathbf{X}_f$ via:
\begin{align}
    \mathbf{X}_f &= \exp\bigl(F(\mathbf{s})\bigr) \, ,
\end{align}
which is then used to compute $\mathbf{D}_f$ and $\mathbf{Y}_f$ analogous to Eq.~(\ref{eq:diag_closed}). Implementation variants for the matrix exponential are detailed in Supp. Sec.~\ref{sec:complexity}.

\begin{lemma}[Asymmetric orthogonality suffices]
For any SVD $\mathbf{A}=\mathbf{U}\mathbf{\Sigma}\mathbf{V}^\top$ with $\det(\mathbf{U})=-1$, jointly flipping the sign of one corresponding column of $\mathbf{U}$ and $\mathbf{V}$ yields an equivalent factorisation $\mathbf{A}=\tilde{\mathbf{U}}\mathbf{\Sigma}\tilde{\mathbf{V}}^\top$ with $\tilde{\mathbf{U}}\in\mathrm{SO}(m)$. A left factor in $\mathrm{SO}(m)$, expressible as $\exp(\mathbf{S})$ for skew-symmetric $\mathbf{S}$, therefore always exists, while the corresponding right factor may lie in the determinant $(-1)$ component of $\mathrm{O}(n)$. Full proof in Supp.~Sec.~\ref{sec:asymmetric}.
\end{lemma}

This justifies treating $\mathbf{X}_f$ and $\mathbf{Y}_f$ asymmetrically during training: $\mathbf{X}_f$ is orthogonal by construction, while $\mathbf{Y}_f$ is only encouraged towards orthogonality through a soft penalty. For square $\mathbf{A}$, $\mathbf{X}_f\mathbf{D}_f\mathbf{Y}_f^\top=\mathbf{X}_f\mathbf{X}_f^\top\mathbf{A}=\mathbf{A}$ holds for any orthogonal $\mathbf{X}_f$, so downstream accuracy rests on the leading columns of $\mathbf{X}_f$. The mathematical properties on singular values and right-factor orthogonality are therefore left as soft constraints using the following learning objectives. (1) A singular value loss:
\begin{align}
    \mathcal{L}_{\sigma}=\left\lVert \mathrm{pdiag}^{-1}(\mathbf{\Sigma}) - \mathrm{pdiag}^{-1}(\mathbf{D}_f)\right\rVert_2^2 \label{eq:loss:sigma}
\end{align}
(2) Given the operator \mbox{$\mathrm{off}(\mathbf{Z})=\mathbf{Z}-\mathrm{diag}\left(\mathrm{diag}^{-1}(\mathbf{Z})\right)$}, an off-diagonal zero loss is defined as:
\begin{align}
    \mathcal{L}_{\mathrm{off}}=\left\lVert\mathrm{off}\left(\mathbf{X}_f^\top\mathbf{A}\mathbf{A}^\top\mathbf{X}_f\right)\right\rVert_F^2 + \left\lVert\mathrm{off}\left(\mathbf{Y}_f^\top\mathbf{A}^\top\mathbf{A}\mathbf{Y}_f\right)\right\rVert_F^2\,
    \label{eq:loss:off}
\end{align}
(3) An orthogonality penalty on $\mathbf{Y}_f$:
\begin{align}
    \mathcal{L}_{\text{ort}}=\left\lVert \mathbf{Y}_f\mathbf{Y}_f^{\top}-\mathbf{I}\right\rVert_2^2 \label{eq:loss:ort}
\end{align}
where $\mathbf{I}\in\mathbb{R}^{n\times n}$ is an identity matrix. These losses are aggregated for a total loss given by
\begin{align}
    \mathcal{L}_{T}=
    \lambda_1 \mathcal{L}_{\sigma} + 
    \lambda_2 \mathcal{L}_{\text{off}} + 
    \lambda_3 \mathcal{L}_{\text{ort}}
    \label{eq:loss:total}
\end{align}
where $(\lambda_1, \lambda_2, \lambda_3)$ are the regularization constants.

\paragraph{End-to-end (E2E) learning}
Many applications of the SVD involve reconstructing a filtered version $\mathbf{M}$ of an input matrix $\mathbf{A}$. For instance, the Lu-Chipman decomposition in Mueller matrix imaging produces $\mathbf{M}=\mathbf{U}\mathbf{D}\mathbf{V}^\top$ after factorisation where $\mathbf{D}=\mathrm{diag}(\mathbf{d})$ and $\mathbf{d}\in\{-1,1\}^m$~\cite{lu1996interpretation}. Scenarios like this offer the opportunity to learn the overarching objective of an SVD in an end-to-end~(E2E) fashion by a function $E_\theta: \mathbb{R}^{m\times n}\mapsto \mathbb{R}^{s\times t}$. The loss then simplifies to:
\begin{align}
    \mathcal{L}_{\text{E2E}}\left(\mathbf{M}, \mathbf{A}\right)=\left\lVert \mathbf{M} - E_\theta(\mathbf{A})\right\rVert_2^2\label{eq:loss:end2end}
\end{align}
as the distance between the network's prediction and final reconstruction matrix $\mathbf{M}$ for which the SVD is used in the processing chain. %
A key advantage of this end-to-end approach is that operations beyond the SVD can be learned implicitly by $E_\theta(\mathbf{A})$, potentially improving computational efficiency. However, this comes at the cost of explicit factor recovery and interpretability, and risks poor generalisation from a potential inability to learn underlying patterns.

\paragraph{Neural module architectures}
Each learned model variant employs an MLP with a single hidden layer of size 
$h=\max(64,m/2)$, where $m$ determines input and output dimensions. To emulate iterative convergence of classical SVD solvers, the MLP is optionally augmented with a three-step RNN that refines input matrix $\mathbf{A}$, where gated recurrent units (GRUs) are used due to balanced accuracy and efficiency~\cite{Foucault:GRU:21,HASSAAN2025112534}.

\section{Experiments}
\label{sec:exp}
\subsection{Datasets and assessment}

The selected datasets represent complementary SVD workload regimes. The MNIST dataset provides matrices of intermediate size, the polarimetric dataset requires performing a very large number of small decompositions, and the ultrasound dataset involves fewer but substantially larger matrices. Together, these cases capture the range of memory and performance trade-offs targeted by the proposed method.
\paragraph{MNIST denoising data} As a controlled benchmark with known ground truth, a denoising task is constructed from MNIST handwritten digits ($28\times28$ grayscale). A $10\times10$ tile grid assembles 100 randomly permuted digits into a single $280\times280$ matrix, providing an intermediate scale regime. Each tile receives independent Gaussian noise ($\sigma=0.2$), with per-epoch D4 dihedral transforms and a random background offset drawn from $[0, 0.5]$ applied to clean tiles \emph{before} noise injection; tile order is shuffled per sample to prevent memorisation and penalise all-zero matrices.

\paragraph{Polarimetric data} MMP is an imaging technique with growing diagnostic relevance in medical imaging~\cite{sampaio2023muller,moriconi2024near}. 
For visualization, the Mueller matrix is typically decomposed via SVD following Lu and Chipman~\cite{lu1996interpretation}, enabling interpretation of polarimetric features. 
In neuropathology, the per-pixel $i$ azimuth angle $\varphi_i=\sfrac{1}{2}\tan^{-1}(\sfrac{M_i^{(2,4)}}{M_i^{(4,3)}})$ derived from $\mathbf{M}_i$ in Eq.~(\ref{eq:loss:end2end}) encodes fiber orientation, a biomarker for tumor identification~\cite{Hahne:25:OpEx}. Although near real-time MMP has been demonstrated~\cite{moriconi2024near}, instant visualization of Lu–Chipman features remains challenging.

The SVD algorithms evaluated on the NeuroPathoPol (NPP) dataset, with the public mouse uterine cervix (MUC) data~\cite{novikova:24:mouse,pogudingleb:25} reserved for domain-transfer analysis. NPP comprises \mbox{$388\times516$-pixel} images, each pixel associated with a $4\times4$ Mueller matrix \smash{$\mathbf{B}_i = B_i^{(u,v)}$}. Training and validation use 12 tumor and 9 non-tumor brain samples, while testing employs 5 tumor and 3 healthy samples. %
For each pixel $i$, a normalised input matrix $\mathbf{A}_i\in\mathbb{R}^{3\times3}$ is derived from $\mathbf{B}_i$ as its bottom-right $3\times3$ block scaled by the $(1,1)$ entry, \smash{$\mathbf{A}_i=\mathbf{B}_i^{(2:4,\,2:4)}/B_i^{(1,1)}$}. Data augmentation through random rotation and flipping increases the diversity of training samples~\cite{Hahne:25:TiP}.

\paragraph{Ultrasound data} Many imaging applications employ the SVD for background removal~\cite{Bo2015}. In ULM, the SVD is a key processing step to suppress static reflectors (e.g., bone) and isolate microbubble flow, commonly realised through the Karhunen–Lo\`eve transform (KLT)~\cite{demene2015spatiotemporal,baranger2018adaptive,heiles2022pala}. 
Given a Casorati matrix \mbox{$\mathbf{B}_c\in\mathbb{R}^{h\times n}$} constructed from $n$ frames with $h$ pixels, each $\mathbf{B}_c$ forms the symmetric matrix $\mathbf{A}=\mathbf{B}_c^\top\mathbf{B}_c$, whose left factor yields the clutter filter $\mathbf{B}_c\mathbf{X}_{\mathcal{S}}\mathbf{X}_{\mathcal{S}}^\top$ with $\mathcal{S}=\{5,\ldots,n-1\}$; only $\mathbf{X}$ enters this filter, and the analytic generator vanishes for symmetric $\mathbf{A}$ ($\mathbf{X}=\mathbf{I}$). %
Evaluation uses publicly available \textit{in vivo} ULM data from the PALA study~\cite{heiles2022pala}, comprising rat brain perfusion sequences acquired with a 15.6~MHz linear probe (128~elements, 0.1~mm pitch) at 1~kHz frame rate with multi-angle plane waves. B-mode data from \textit{rat 18} is employed, recorded in sequences of 800 frames, each split into three contiguous stacks of 250 frames. Stacks of the first 20 sequences are used for training and validation (0.9 split), and the 15 stacks of the next 5 sequences for testing. Each B-mode image contains $h=143\times167$ pixels, split into 11 spatial blocks that yield eleven $\mathbf{A}_i\in\mathbb{R}^{250\times250}$ per stack of $n=250$ frames.

\paragraph{Metrics}
\label{ssec:metrics}
Reconstruction quality and decomposition accuracy are reported using both perceptual and numerical metrics. The structural similarity index (SSIM) measures local luminance, contrast, and structure using an 11-pixel Gaussian window. For matrix-level assessment, a relative Frobenius error is defined as $\|\mathbf{D}\|_F^{\text{rel}}=\|\mathbf{D}-\mathbf{D}^{(gt)}\|_F/\|\mathbf{D}^{(gt)}\|_F$ and $\|\mathbf{Y}\|_F^{\text{ort}}=\|\mathbf{Y}\mathbf{Y}^{\top}-\mathbf{I}\|_F/\|\mathbf{I}\|_F$ from the Frobenius norm $\|\cdot\|_F$ with $^{(gt)}$ denoting the ground truth (GT) and $\mathbf{I}$ the identity matrix. 
\paragraph{Training and testing}
Models are trained using AdamW with gradient clipping and cosine-annealed learning rates converging to zero. Weights are initialised with Xavier (Glorot) uniform for linear layers, biases drawn from \(\mathcal{N}(0,10^{-6})\), and batch‑norm weights/biases set to \(1\) and \(0\), respectively. MNIST data uses an Nvidia T500 mobile GPU (in 341 MHz mode for inference), batch size 32, 200 epochs, initial learning rate $1\mathrm{e}{-4}$, with best checkpoint by lowest validation SSIM. Polarimetric data uses an Nvidia RTX 4090, batch size 776, 40 epochs, initial learning rate $1\mathrm{e}{-3}$, with best checkpoint by lowest \(\mathcal{L}_{\textrm{ort}}\) from Eq.~(\ref{eq:loss:ort}) (or validation \(\lVert\varphi_i-\varphi_i^{(gt)}\rVert_2^2\) when training with \(\mathcal{L}_{\text{E2E}}\)). Ultrasound data uses an Nvidia RTX 5090, batch size 1, 200 epochs, initial learning rate $1\mathrm{e}{-5}$, with best checkpoint by lowest \(\mathcal{L}_T\) from Eq.~(\ref{eq:loss:total}) (or validation \(\mathcal{L}_{\text{E2E}}\) when training with \(\mathcal{L}_{\text{E2E}}\)). Runtimes are measured on identical batched inputs and hardware using synchronized wall-clock timing, including method-specific overheads.

\paragraph{Comparison methods}
The proposed framework is evaluated against classical and learning-based SVD techniques: (i) Exact SVD, a reference implementation via cuSOLVER~\cite{nvidia2025cusolver} using one-sided Jacobi SVD for small matrices ($m,n\le32$) and traditional decomposition otherwise~\cite{golub1971singular}, serving as the numerical baseline for accuracy; (ii) SV-Learn, an MLP-based predictor of singular values only, without singular vector reconstruction~\cite{xu2022sv}; (iii) end-to-end (E2E) networks, as in L-SVD~\cite{brune2019learned,cole2022learned}, trained to directly reconstruct SVD-filtered outputs; (iv) rSVD, a stochastic low-rank approximation~\cite{halko2011finding} (PyTorch), with default settings except oversampling ($q$ capped at $\min(m,n)$, set to the SVD-optimal rank for MNIST; for ULM, $q=6$ estimates the clutter subspace, which is subtracted) and zero-padding factors to full size for scoring; and (v) \texttt{eigh}, cuSOLVER's symmetric eigensolver applied to the ULM Gram matrix. For ablation, each learned model is extended with an RNN, and the analytical and UNN variants are included. 

\subsection{Results}
This section outlines general experimental trends followed by domain-specific outcomes. Numerical evaluations for all three tasks are summarised in Tables~\ref{tab:metrics:mnist}, \ref{tab:metrics:npp},~and~\ref{tab:metrics:ulm}, with qualitative comparisons in Figs.~\ref{fig:mnist},~\ref{fig:azimuth}~and~\ref{fig:ulm}. Overall, RNN-enabled models deliver the highest fidelity among learned models, while LieNN best balances accuracy and orthogonality, reducing runtime relative to the exact SVD by $6$--$10\times$ (MNIST), $2.6$--$3.7\times$ (MMP) and $67\times$ (ULM) at little overhead relative to~UNN. 
\paragraph{Results for MNIST denoising ($\mathbf{A}_i\in\mathbb{R}^{280\times280}$):} %
Table~\ref{tab:metrics:mnist} reports quantitative results for the MNIST patchwork denoising benchmark. LieNN with RNN achieves the highest PSNR (17.90 dB), marginally exceeding classical SVD truncation (17.65 dB, 0.51) at roughly $6\times$ the speed. Without RNN, LieNN matches SVD while being nearly $10\times$ faster. This advantage over the exact SVD reflects implicit regularisation: the learned basis attenuates noise components that the exact SVD retains. LieNN's optimal truncation rank ($k=96$--$97$) is substantially higher than SVD's ($k=38$), suggesting the learned basis concentrates signal energy differently. SV-Learn and E2E without RNN perform poorly and the analytic variant fails. LieNN's advantage holds across a wide range of $k$ (Supp.~Fig.~\ref{fig:mnist:sweep}).

\begin{table*}[!h]
    \caption{\textbf{MNIST test set~($\mathbf{A}_i\in\mathbb{R}^{280\times 280}$) for denoising.} Mean~$\pm$~std across frames ($\text{--}$: not applicable; std omitted when 0 or mean $>10^1$). \textbf{Bold} best, \colorbox{gray!15}{shaded} second-best non-GT score. Time per matrix on an Nvidia T500. $^\dagger$9-term Taylor map outside accuracy range.\label{tab:metrics:mnist}}
    \centering
    \resizebox{.72\textwidth}{!}{%
    \begin{tabular}{lclS[table-format=2.2(2),detect-all]lllS[table-format=3.2(2),detect-all]}
\toprule
Method & RNN & $k$ & {$\overline{\mathrm{PSNR}}$~[dB]$\uparrow$} & {$\overline{\|\mathbf{X}\|_F^{\text{ort}}}\downarrow$} & {$\overline{\|\mathbf{D}\|_F^{\text{rel}}}\downarrow$} & {$\overline{\|\mathbf{Y}\|_F^{\text{ort}}}\downarrow$} & {\text{Time}~\text{[ms]}$\downarrow$} \\
\midrule
Noisy input & \ding{55} & -- & 15.40(0.05) & {--} & {--} & {--} & 0.02 \\
\midrule
Exact SVD & \ding{55} & $38$ & 17.65(0.08) & \bfseries\boldmath $0.00$ & \bfseries\boldmath $0.00$ & \bfseries\boldmath $0.00$ & 64.23(1.20) \\
\texttt{eigh} & \ding{55} & $38$ & 17.65(0.08) & \bfseries\boldmath $0.00$ & \bfseries\boldmath $0.00$ & \bfseries\boldmath $0.00$ & 44.28(0.61) \\
rSVD & \ding{55} & $38$ & 17.65(0.07) & $0.87$ & $0.33$ & \cellcolor{gray!15} $0.87$ & 19.35(0.34) \\
SV-Learn & \ding{55} & $1$ & 8.07(0.14) & {--} & $0.11$ & {--} & \bfseries\boldmath 0.67(0.03) \\
E2E & \ding{55} & -- & 12.85(0.13) & {--} & {--} & {--} & \cellcolor{gray!15} 1.28(0.09) \\
\midrule
Analytic$^\dagger$ & \ding{55} & $280$ & 2.97(0.33) & $3.57 \times 10^{8}$ & $3.43 \times 10^{4}$ & $1.43 \times 10^{1}$ & 3.17(0.07) \\
UNN & \ding{55} & $4$ & 11.78(0.24) & $0.94$ & $0.07 \pm 0.03$ & $1.15 \times 10^{1}$ & 3.30(0.08) \\
LieNN & \ding{55} & $96$ & \cellcolor{gray!15} 17.78(0.10) & $0.21 \pm 0.02$ & $0.18 \pm 0.01$ & $1.30 \pm 0.03$ & 6.59(0.13) \\
\midrule
E2E & \ding{51} & -- & 12.55(0.15) & {--} & {--} & {--} & 5.41(0.05) \\
UNN & \ding{51} & $2$ & 12.08(0.14) & $0.96$ & \cellcolor{gray!15} $0.02 \pm 0.01$ & $1.14 \times 10^{1}$ & 7.33(0.05) \\
LieNN & \ding{51} & $97$ & \bfseries\boldmath 17.90(0.07) & \cellcolor{gray!15} $0.15$ & $0.20$ & $1.26 \pm 0.02$ & 10.33(0.15) \\
\bottomrule
\end{tabular}

    }
\end{table*}
\begin{figure*}[!ht]
    \centering
    \ifJournalOption
\else
        \includegraphics[width=\linewidth]{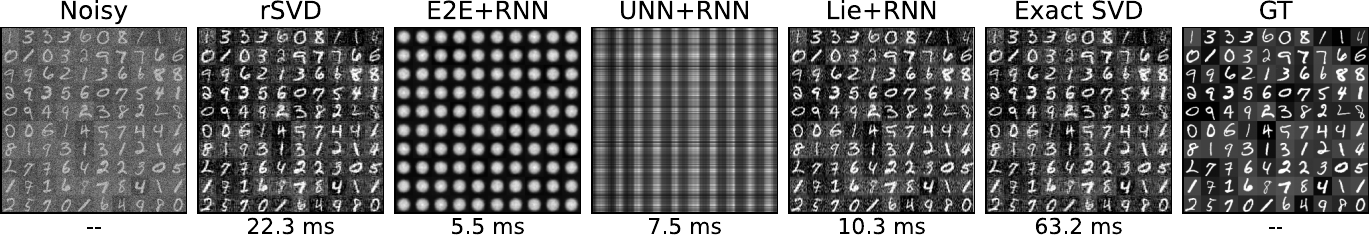}
        \caption{\textbf{MNIST patchwork denoising comparison} ($\mathbf{A}_i\in\mathbb{R}^{280\times280}$; noise $\sigma=0.2$). Columns depict reconstructions by different methods with low-cost Nvidia T500 GPU runtimes (one matrix per call) shown below.\label{fig:mnist}}
    \fi
\end{figure*}
\paragraph{Results for MMP ($\mathbf{A}_i\in\mathbb{R}^{3 \times 3}$):} %
Neural components learn to approximate the algebraic operations of the Lu–Chipman decomposition~\cite{lu1996interpretation}, suggesting the potential to replace symbolic computations with accelerated abstractions. End-to-end modules most effectively cover these decomposition and filter steps but produce no explicit matrix factors, limiting their utility for spectral analysis. %
The rSVD reproduces the exact output, yet it is three orders of magnitude slower than LieNN. This is inherent to randomised SVD in this regime: with $q=\min(m,n)=3$ the random sketch spans the full space, so after its power iterations rSVD still computes an exact SVD of the same size and cannot be faster than the exact solver. 
Large batches of small full-rank matrices, which LieNN decomposes $2.6\times$ faster than the exact SVD, thus lie outside the reach of randomised SVD. A closed-form $3\times3$ solver~\cite{mcadams2011computing} is faster still at near-exact accuracy (2.39 vs.\ 14.19~ms, Tab.~\ref{tab:metrics:npp}) but, unlike LieNN, does not extend beyond $3\times3$ matrices. %
Figure~\ref{fig:azimuth} shows the per-pixel $\varphi_i\in\mathbb{R}$ azimuth images for the two polarimetric domains. The analytic model approximates GT well despite appearing noisy, showing a deterministic first-order SVD approximation can be competitive under tight budgets. Among learnable models without RNNs, LieNN performs best, while the end-to-end model shows a consistent offset. With RNNs, all learnable models closely track the Exact SVD, with only minor deviations in the UNN. 
%
\begin{table*}[!h]
    \caption{\textbf{NPP test set~($\mathbf{A}_i\in\mathbb{R}^{3 \times 3}$) for Lu-Chipman decomposition.} Mean~$\pm$~std across frames ($\text{--}$ for not applicable; std omitted when 0). \textbf{Bold} best, \colorbox{gray!15}{shaded} second-best non-Exact-SVD score. Time per frame ($388\times516$ matrices in one batch) on an Nvidia RTX 4090. \label{tab:metrics:npp}}
    \centering
    \resizebox{.71\textwidth}{!}{%
    \begin{tabular}{l c S[table-format=2.2(2),detect-all] S[table-format=2.2(2),detect-all] S[table-format=2.2(2),detect-all] S[table-format=2.2(2),detect-all] S[table-format=2.2(2),detect-all] S[table-format=2.2(2),detect-all] S[table-format=2.2(2),detect-all]}
\toprule
Method & RNN & $\overline{\mathrm{SSIM}(\mathbf{M})}\uparrow$ & $\overline{\|\mathbf{X}\|_F^{\text{ort}}}\downarrow$ & $\overline{\|\mathbf{D}\|_F^{\text{rel}}}\downarrow$ & $\overline{\|\mathbf{Y}\|_F^{\text{ort}}}\downarrow$ & $\text{Time}~\text{[ms]}\downarrow$ \\
\midrule
Exact SVD
& \ding{55} & 1.00(0.00) & 0.00(0.00) & 0.00(0.00) & 0.00(0.00) & 36.98(1.81) \\
rSVD
& \ding{55} & \bfseries 1.00(0.00) & \bfseries 0.00(0.00) & \bfseries 0.00(0.00) & \bfseries 0.00(0.00) & {2.12e4±92.42} \\
Closed-form
& \ding{55} & \bfseries 1.00(0.00) & \bfseries 0.00(0.00) & \bfseries 0.00(0.00) & \cellcolor{gray!15} 0.02(0.01) & 2.39 \\
SV-Learn
& \ding{55} & 0.65(0.08) & \text{--} & 0.50(0.01) & \text{--} & \cellcolor{gray!15} 1.89(0.00) \\
E2E
& \ding{55} & 0.75(0.07) & \text{--} & \text{--} & \text{--} & \bfseries 0.66(0.00) \\
\midrule
Analytic & \ding{55} & 0.66(0.08) & \bfseries 0.00(0.00) & 0.09(0.04) & 0.46(0.08) & 8.48(0.02) \\
UNN & \ding{55} & 0.82(0.05) & 0.13(0.04) & 0.03(0.01) & 0.23(0.07) & 9.85(0.01) \\
LieNN & \ding{55} & 0.83(0.06) & \bfseries 0.00(0.00) & 0.03(0.01) & 0.24(0.07) & 9.93(0.02) \\
\midrule
E2E & \ding{51} & \cellcolor{gray!15} 0.98(0.01) & \text{--} & \text{--} & \text{--} & 4.91(0.00) \\
UNN & \ding{51} & 0.85(0.04) & 0.14(0.03) & 0.02(0.00) & 0.08(0.02) & 13.81(0.01) \\
LieNN & \ding{51} & 0.94(0.03) & \bfseries 0.00(0.00) & \cellcolor{gray!15} 0.01(0.00) &  0.10(0.03) & 14.19(0.02) \\
\bottomrule
\end{tabular}
    }
\end{table*}
\begin{table*}[!h]
    \centering
    \caption{\textbf{MUC data~($\mathbf{A}_i\in\mathbb{R}^{3 \times 3}$) for domain gap analysis.} Results marked as in Table~\ref{tab:metrics:npp} with time per frame of $600\times700$ matrices. \label{tab:metrics:mouse}}
    \resizebox{.70\textwidth}{!}{%
    \begin{tabular}{l c S[table-format=2.2(2),detect-all] S[table-format=2.2(2),detect-all] S[table-format=2.2(2),detect-all] S[table-format=2.2(2),detect-all] S[table-format=2.2(2),detect-all] S[table-format=2.2(2),detect-all] S[table-format=2.2(2),detect-all]}
\toprule
Method & RNN & $\overline{\mathrm{SSIM}(\mathbf{M})}\uparrow$ & $\overline{\|\mathbf{X}\|_F^{\text{ort}}}\downarrow$ & $\overline{\|\mathbf{D}\|_F^{\text{rel}}}\downarrow$ & $\overline{\|\mathbf{Y}\|_F^{\text{ort}}}\downarrow$ & $\text{Time}~\text{[ms]}\downarrow$ \\
\midrule
Exact SVD
& \ding{55} & 1.00(0.00) & 0.00(0.00) & 0.00(0.00) & 0.00(0.00) & 98.80(0.00) \\
rSVD
& \ding{55} & \bfseries 1.00(0.00) & \bfseries 0.00(0.00) & \bfseries 0.00(0.00) & \bfseries 0.00(0.00) & {4.33e4}\hfill \\ 
Closed-form
& \ding{55} & \bfseries 1.00(0.00) & \bfseries 0.00(0.00) & \bfseries 0.00(0.00) & \bfseries 0.00(0.00) & 4.12 \\
SV-Learn
& \ding{55} & 0.41(0.00) & \text{--} & 0.55(0.00) & \text{--} & \cellcolor{gray!15} 3.14(0.00) \\
E2E
& \ding{55} & 0.50(0.00) & \text{--} & \text{--} & \text{--} & \bfseries 1.21(0.00) \\
\midrule
E2E & \ding{51} & \cellcolor{gray!15} 0.89(0.00) & \text{--} & \text{--} & \text{--} & 10.29(0.00) \\
UNN & \ding{51} & 0.68(0.00) & 0.38(0.00) & 0.11(0.00) & \cellcolor{gray!15} 0.16(0.00) & 28.37(0.00) \\
LieNN & \ding{51} & 0.83(0.00) & \bfseries 0.00(0.00) & \cellcolor{gray!15} 0.07(0.00) & 0.25(0.00) & 28.96(0.00) \\
\bottomrule
\end{tabular}
    }
\end{table*}
\begin{figure*}[!h]
    \centering
    \begin{minipage}[c]{.93\linewidth}
        \begin{subfigure}[c]{\linewidth}
            \ifJournalOption
\else
                \includegraphics[width=\linewidth]{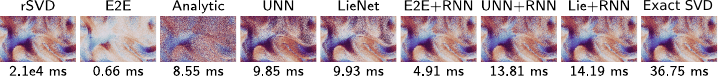}
            \fi
            \caption{Azimuth frames (388 × 516 pixels) from the NPP test set}\label{fig:azimuth:a}
        \end{subfigure}
        \begin{subfigure}[c]{\linewidth}
            \ifJournalOption
\else
                \vspace{0.5em}\includegraphics[width=\linewidth]{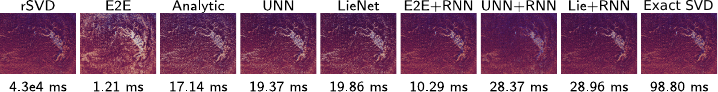}
            \fi
            \caption{Azimuth frames (600 × 700 pixels) from the MUC data for domain shift analysis}\label{fig:azimuth:b}
        \end{subfigure}
    \end{minipage}
    \hfill
    \begin{minipage}[c]{.0375\linewidth}
        \includegraphics[width=\linewidth]{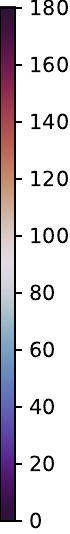}
    \end{minipage}
    \caption{\textbf{MMP benchmark comparison after Lu-Chipman decomposition~($\mathbf{A}_i\in\mathbb{R}^{3 \times 3}$).} The images show the per-pixel azimuth $\varphi_i$ by the various models with Nvidia RTX 4090 GPU runtimes below. \subref{fig:azimuth:a} shows an example from the NPP test set (more in Supp.~Fig.~\ref{fig:azimuth:remaining}); \subref{fig:azimuth:b} shows a larger frame from the MUC data, using the same models to assess adaptation across instrument domain gaps.\label{fig:azimuth}}
\end{figure*}
As MMP-trained model variants predict per-pixel components without spatial context, they transfer robustly to a separate instrument with a different optical setup (Table~\ref{tab:metrics:mouse}, Fig.~\ref{fig:azimuth:b}), retaining their relative performance without retraining at reduced absolute accuracy.
%
\paragraph{Results for ULM using the PALA data ($\mathbf{A}_i\in\mathbb{R}^{250\times250}$):} %
Table~\ref{tab:metrics:ulm} reports results for ULM spatio-temporal decomposition. Among learned models, LieNN with RNN achieves the highest SSIM with a near-orthogonal left factor ($\overline{\|\mathbf{X}\|_F^{\text{ort}}}=0.01$ vs.\ 0.63 for UNN); its larger $\mathbf{D}$ and $\mathbf{Y}$ errors do not affect the output, as the ULM filter uses $\mathbf{X}$ only. Unlike in MMP, E2E is the least accurate learned model. \texttt{eigh} and rSVD match the Exact SVD, yet LieNN filters a stack 35$\times$ and 4.5$\times$ faster, respectively, as it amortises its batched matrix products across jointly filtered stacks.
\par

\begin{table*}[!ht]
    \centering
    \caption{\textbf{ULM results for spatio-temporal filtering~($\mathbf{A}_i\in\mathbb{R}^{250 \times250}$).} Mean~$\pm$~std across test stacks ($\text{--}$: not applicable; std omitted when 0). \textbf{Bold} best, \colorbox{gray!15}{shaded} second-best non-Exact-SVD score. $^\dagger$$\mathbf{X}=\mathbf{I}$ for symmetric $\mathbf{A}$. Time per stack with the 15 test stacks filtered jointly. 
    \label{tab:metrics:ulm}}
    \resizebox{.75\textwidth}{!}{%
    \begin{tabular}{l S[table-format=2.2(2),detect-all] S[table-format=2.2(2),detect-all] S[table-format=2.2(2),detect-all] S[table-format=2.2(2),detect-all] S[table-format=2.2(2),detect-all] S[table-format=2.2(2),detect-all] S[table-format=2.2(2),detect-all] S[table-format=2.2(2),detect-all]}
\toprule
Method & RNN & $\overline{\mathrm{SSIM}(\mathbf{M})}\uparrow$ & $\overline{\|\mathbf{X}\|_F^{\text{ort}}}\downarrow$ & $\overline{\|\mathbf{D}\|_F^{\text{rel}}}\downarrow$ & $\overline{\|\mathbf{Y}\|_F^{\text{ort}}}\downarrow$ & $\text{Time}~\text{[ms]}\downarrow$ \\
\midrule
Exact SVD
& \ding{55} & 1.00(0.00) & 0.00(0.00) & 0.00(0.00) & 0.00(0.00) & 60.59(0.01) \\
\texttt{eigh}
& \ding{55} & \bfseries 1.00(0.00) & \bfseries 0.00(0.00) & \bfseries 0.00(0.00) & \bfseries 0.00(0.00) & 31.58(0.00) \\
rSVD
& \ding{55} & \bfseries 1.00(0.00) & 0.99(0.00) & \bfseries 0.00(0.00) & \cellcolor{gray!15} 0.99(0.00) & 4.12(0.00) \\
Analytic$^\dagger$ & \ding{55} & 0.70(0.02) & \bfseries 0.00(0.00) & 1.37(0.00) & 15.78(0.00) & \cellcolor{gray!15} 0.68(0.00) \\
\midrule
E2E & \ding{51} & 0.90(0.05) & \text{--} & \text{--} & \text{--} & \bfseries 0.56(0.00) \\
UNN & \ding{51} & 0.92(0.01) & 0.63(0.00) & \cellcolor{gray!15} 0.86(0.00) & 15.36(1.18) & 0.75(0.00) \\
LieNN & \ding{51} & \cellcolor{gray!15} 0.95(0.01) & \cellcolor{gray!15} 0.01(0.00) & \cellcolor{gray!15} 0.86(0.00) & 20.89(9.62) & 0.91(0.00) \\
\bottomrule
\end{tabular}
}
\end{table*}
\begin{figure*}[!h]
    \centering
    \ifJournalOption
\else
        \begin{minipage}[c]{.915\linewidth}
            \includegraphics[width=\linewidth]{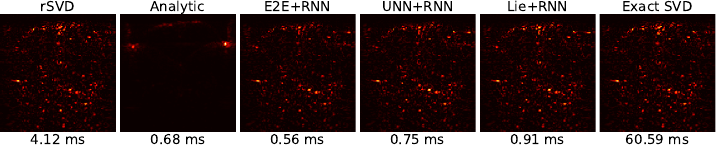}
        \end{minipage}
        \hfill
        \begin{minipage}[c]{.04725\linewidth}
            \ifJournalOption
                \includegraphics[width=\linewidth]{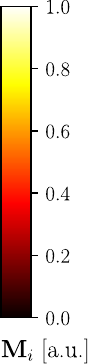}
            \else
                \includegraphics[width=\linewidth]{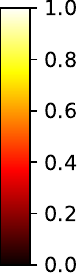}
            \fi
        \end{minipage}
    \fi
    \caption{\textbf{ULM clutter removal ($\mathbf{A}_i\in\mathbb{R}^{250\times 250}$).} Images show the \mbox{B-mode} intensity after spatio-temporal filtering (colour-encoded for better visibility). Columns represent methods and bottom numbers the per-stack Nvidia RTX 5090 computation time of Tab.~\ref{tab:metrics:ulm}. \label{fig:ulm}}
\end{figure*}
%
Qualitative inspection of a ULM filtered B-mode frame in Fig.~\ref{fig:ulm} shows that
among learned models, LieNN produces the closest match to the Exact SVD, retaining microbubble contrast and suppressing clutter while operating $67\times$ faster than the Exact SVD and $4.5\times$ faster than rSVD at sub-millisecond cost per stack, which brings online ULM clutter filtering within reach. Supp.~Fig.~\ref{fig:times} and Sec.~\ref{supp:sec:ulm} show how this advantage scales with matrix size, stack length and batching.

\section{Conclusion}
\label{sec:conclusion}
This work introduced a Lie-informed differentiable approximation of SVD designed for massive GPU parallelism. By enforcing orthogonality analytically via skew-symmetric parametrization and the matrix exponential while relaxing singular-value exactness, it provides an orthogonal basis at faster computation times. Across MNIST denoising, polarimetric, and ultrasound benchmarks, Lie-constrained models are the most accurate learned factorisations while substantially reducing computation relative to the exact SVD, including large batches of small matrices, for which randomised SVD offers no saving. The pixel-wise formulation yields strong domain transfer without retraining. As the factorisation consists of differentiable batched matrix products only, it scales with GPU parallelism and can serve as a trainable decomposition layer within learning pipelines. Future extensions include additional constraints for improved singular-value and left-factor fidelity, and complex-valued implementations.

\bibliographystyle{plainnat}
\bibliography{main}

\newpage

\clearpage
\onecolumn  
\begin{center}
    {\Large\bfseries Supplementary Material}
\end{center}
\vspace{1em}

\appendix
\section{Mathematical foundations}

\subsection{Orthogonality via matrix exponential}

\paragraph{Theorem (Skew-Symmetric Exponential is Orthogonal)} Let $\mathbf{S}\in\mathfrak{so}(m)$ be a real skew-symmetric matrix (i.e., $\mathbf{S}^\top = -\mathbf{S}$). Then the matrix exponential
$$
\exp(\mathbf{S}) = \sum_{k=0}^{\infty} \frac{\mathbf{S}^k}{k!}
$$
is an orthogonal matrix, meaning that
$$
\exp(\mathbf{S})^\top \exp(\mathbf{S}) = \mathbf{I}.
$$

\paragraph{Proof} As $\mathbf{S}$ is skew-symmetric, we have \mbox{$\mathbf{S}^\top = -\mathbf{S}$}. Consider the matrix exponential:
$$
\exp(\mathbf{S}) = \mathbf{I} + \mathbf{S} + \frac{\mathbf{S}^2}{2!} + \frac{\mathbf{S}^3}{3!} + \cdots.
$$
Taking the transpose, and using the linearity of the transpose along with the fact that $(\mathbf{S}^k)^\top = (\mathbf{S}^\top)^k$, we obtain
$$
\exp(\mathbf{S})^\top = \mathbf{I} + \mathbf{S}^\top + \frac{(\mathbf{S}^\top)^2}{2!} + \frac{(\mathbf{S}^\top)^3}{3!} + \cdots.
$$
Substitute $\mathbf{S}^\top = -\mathbf{S}$:
$$
\exp(\mathbf{S})^\top = \mathbf{I} - \mathbf{S} + \frac{\mathbf{S}^2}{2!} - \frac{\mathbf{S}^3}{3!} + \cdots = \exp(-\mathbf{S}).
$$

Now, consider the product:
$$
\exp(\mathbf{S})^\top \exp(\mathbf{S}) = \exp(-\mathbf{S}) \exp(\mathbf{S}).
$$
Using the property of the exponential function on commuting matrices,
$$
\exp(-\mathbf{S}) \exp(\mathbf{S}) = \exp(-\mathbf{S} + \mathbf{S}) = \exp(\mathbf{0}) = \mathbf{I}.
$$
Thus, $\exp(\mathbf{S})^\top \exp(\mathbf{S}) = \mathbf{I}$, which shows that $\exp(\mathbf{S})$ is orthogonal. 

\paragraph{Corollary} This guarantees that $\mathbf{X} = \exp(\mathbf{S}) \in \mathrm{SO}(m)$ for any skew-symmetric $\mathbf{S} \in \mathfrak{so}(m)$.

\subsection{Special orthogonal group}
\label{sec:specialgroup}

The matrix exponential of a skew-symmetric matrix produces an element of the special orthogonal group,
\begin{equation*}
	\exp(\mathbf{S}) \in \mathrm{SO}(m) \subset \mathrm{O}(m),
\end{equation*}
which enforces the constraint $\det(\mathbf{X}) = 1$. In contrast, the full orthogonal group also contains matrices with negative determinant, i.e.\ $\det(\mathbf{X}') \in \{1,-1\}$ for $\mathbf{X}' \in \mathrm{O}(m)$.

In some applications, such as matching the sign structure of a target orthogonal factor, it is necessary to obtain an orthogonal matrix with either determinant. To allow this flexibility while still parameterising $\mathbf{X}$ through a skew-symmetric generator, we introduce a diagonal by
\begin{equation*}
	\mathbf{B} = \operatorname{diag}(\mathbf{b}), \qquad \mathbf{b} \in \{-1,1\}^m,
\end{equation*}
which represents an element of $\mathrm{O}(m)$. The corrected orthogonal matrix is then computed as
\begin{equation*}
	\mathbf{X}' = \mathbf{B} \mathbf{X}.
\end{equation*}
where left-multiplication by $\mathbf{B}$ flips selected column directions of $\mathbf{X}$ and adjusts the determinant when required, while preserving orthogonality:
\begin{equation*}
(\mathbf{X}')^{\top} \mathbf{X}' = \mathbf{X}^{\top} \mathbf{B}^{2} \mathbf{X} = \mathbf{I}.
\end{equation*}
\par
We evaluated the reflection correction during ablations and found that it neither improved approximation quality nor altered the stability of training. Because its effect was neutral and the method functions identically without it, we report the mechanism only in the supplementary to avoid diverting attention from components that materially influence performance. While this correction is theoretically available, we demonstrate hereafter that restricting to $\mathrm{SO}(m)$ is sufficient for valid decomposition.

\subsection{Why Asymmetric Constraints Suffice}
\label{sec:asymmetric}
\paragraph{Lemma (Existence of a Decomposition with $\mathbf{X}\in\mathrm{SO}(m)$)} Consider a matrix $\mathbf{A} \in \mathbb{R}^{m \times n}$ with $\mathbf{A} = \mathbf{X} \mathbf{D} \mathbf{Y}^\top$, where $\mathbf{X}, \mathbf{Y}$ contain the left and right singular vectors of $\mathbf{A}$ and $\mathbf{D}=\operatorname{diag}(\mathbf{d})$ is the (possibly rectangular) diagonal matrix containing the singular values, such that for the column vectors $\mathbf{x}_k, \mathbf{y}_k$ holds
\begin{equation*}
\mathbf{A}\mathbf{A}^\top \mathbf{x}_k = \sigma_k^2 \mathbf{x}_k, \quad
\mathbf{A}^\top\mathbf{A} \mathbf{y}_k = \sigma_k^2 \mathbf{y}_k.
\end{equation*}

Now, if $\mathbf{X}\in\mathrm{O}(m)$ with $\det(\mathbf{X})=-1$, let
$\mathbf{B}_X=\operatorname{diag}(-1,1,\ldots,1)$ and
$\mathbf{B}_Y=\operatorname{diag}(-1,1,\ldots,1)$ denote sign matrices
that flip the same corresponding singular-vector pair, and define
\begin{align*}
\tilde{\mathbf{X}}=\mathbf{X}\mathbf{B}_X, \qquad
\tilde{\mathbf{Y}}=\mathbf{Y}\mathbf{B}_Y.
\end{align*}


These $\tilde{\mathbf{X}}$ and $\tilde{\mathbf{Y}}$ remain orthogonal because they are obtained by right multiplication with orthogonal sign matrices. For the corresponding sign flip, the diagonal structure of $\mathbf{D}$ gives
\[
\mathbf{B}_X\mathbf{D}
\mathbf{B}_Y^\top=\mathbf{D},
\]
and hence
\begin{align*}
	\tilde{\mathbf{X}}\mathbf{D}\tilde{\mathbf{Y}}^\top
	&= \mathbf{X}\mathbf{B}_X\mathbf{D}
	\mathbf{B}_Y^\top\mathbf{Y}^\top \\
	&= \mathbf{X}\mathbf{D}\mathbf{Y}^\top \\
	&= \mathbf{A}.
\end{align*}
The matrices $\tilde{\mathbf{X}}$ and $\tilde{\mathbf{Y}}$ still contain the
singular vectors corresponding to the same singular values, since right
multiplication by the sign matrices only changes the signs of their columns.
In particular, if $\mathbf{x}_k$ and $\mathbf{y}_k$ correspond to
$\sigma_k$, then replacing them by $-\mathbf{x}_k$ and $-\mathbf{y}_k$
preserves their singular-vector relations and leaves $\sigma_k$ unchanged.

Moreover, $\tilde{\mathbf{X}} \in \mathrm{SO}(m)$:
\begin{align*}
\det(\tilde{\mathbf{X}})
&= \det(\mathbf{X})\det(\mathbf{B}_X) \\[4pt]
&= (-1)(-1) = 1.
\end{align*}
Note that this sign flexibility (discussed in \ref{sec:specialgroup}) is a theoretical convenience for the proof; the practical method does not require explicit sign matching.


This lemma motivates the asymmetric design: an equivalent SVD representation always exists with $\mathbf{X}\in\mathrm{SO}(m)$, so enforcing orthogonality on $\mathbf{X}$ via the exponential map does not exclude any valid SVD representation. In contrast, the corresponding right factor may lie in the determinant $(-1)$ component of $\mathrm{O}(n)$ and is therefore not constrained to $\mathrm{SO}(n)$. In the proposed method, $\mathbf{Y}$ is recovered residually and encouraged toward orthogonality through the soft constraint $\mathcal{L}_{\text{ort}}$.

\subsection{Approximation Bounds}

\subsubsection{Gershgorin Bounds on Singular Values}

\paragraph{Gershgorin Circle Theorem} Let $\mathbf{A} \in \mathbb{R}^{p \times p}$ with entries denoted as $a_{ij}$. Then, the eigenvalues are contained in the circles whose radii are bounded by the sums of the matrix entries:
\[
\lambda \in \left\{z \mid |z - a_{ii}| < R_i\right\} \quad \text{where} \quad R_i = \sum_{\substack{j=1 \\ j \neq i}}^p |a_{ij}|.
\]

\paragraph{Bounds on Singular Values via Coherence} Consider the symmetric squared matrix $\mathbf{B} := \mathbf{A}^\top \mathbf{A} \in \mathbb{R}^{n \times n}$ such that $b_{ik} = \sum_{j=1}^n a_{ji} a_{jk} = \langle \mathbf{a}_{.i}, \mathbf{a}_{.k} \rangle$, where $\mathbf{a}_{.i}, \mathbf{a}_{.k}$ are the $i$-th and $k$-th columns of $\mathbf{A}$. Applying the Gershgorin Circle Theorem to $\mathbf{B}$, the eigenvalues of $\mathbf{B}$ (i.e., the squared singular values of $\mathbf{A}$) satisfy
\[
\lambda \in \left\{z \mid |z - \|\mathbf{a}_{.i}\|^2| < S_i\right\},
\]
\[
S_i := \sqrt{n-1} \cdot \mu(\mathbf{A}) \|\mathbf{A}\|_F \|\mathbf{a}_{.i}\|,
\]
where the coherence of $\mathbf{A}$ is given by
\[
\mu(\mathbf{A}) = \max_{i \neq k} \frac{\lvert\langle \mathbf{a}_{.i}, \mathbf{a}_{.k} \rangle\rvert}{\|\mathbf{a}_{.i}\| \|\mathbf{a}_{.k}\|}.
\]

\paragraph{Proof} By the Gershgorin Circle Theorem applied to $\mathbf{B}$, the Gershgorin radius for row $i$ is
\[
R_i^{\mathbf{B}} := \sum_{\substack{k=1 \\ k \neq i}}^n |b_{ik}|.
\]
The diagonal elements satisfy
\[
|b_{ii}| = \left|\sum_{j=1}^n a_{ji}^2\right| = \|\mathbf{a}_{.i}\|^2,
\]
confirming the center of the $i$-th Gershgorin disc. For the off-diagonal radius, using the coherence bound $|\langle \mathbf{a}_{.i}, \mathbf{a}_{.k}\rangle| \le \mu(\mathbf{A})\|\mathbf{a}_{.i}\|\|\mathbf{a}_{.k}\|$:
\begin{align*}
R_i^{\mathbf{B}} &= \sum_{\substack{k=1 \\ k \neq i}}^n |b_{ik}|
= \sum_{\substack{k=1 \\ k \neq i}}^n |\langle \mathbf{a}_{.i}, \mathbf{a}_{.k} \rangle| \\
&\leq \mu(\mathbf{A}) \|\mathbf{a}_{.i}\| \sum_{\substack{k=1 \\ k \neq i}}^n \|\mathbf{a}_{.k}\|
\leq \mu(\mathbf{A}) \|\mathbf{a}_{.i}\| \sqrt{n-1} \|\mathbf{A}\|_F = S_i,
\end{align*}
where the final step applies Cauchy--Schwarz,
\[
\sum_{k=1}^n \|\mathbf{a}_{.k}\| \leq \sqrt{n} \|\mathbf{A}\|_F,
\]
restricted to the $n-1$ terms $k\neq i$. Thus $R_i^{\mathbf{B}} \le S_i$, establishing $S_i$ as a valid Gershgorin radius bound.



\ifJournalOption
\else
    \section{Work–Span Complexity Analysis and Proofs}
    \subsection{Computational complexity}
    \label{sec:complexity}
    The following analysis considers the different SVD algorithms under the work--span model~\cite{blelloch1996programming} (EREW PRAM) for $m \le n$, highlighting its suitability for SIMT (Single Instruction, Multiple Threads) architectures. Unlike classical solvers limited by sequential steps, the Lie-based SVD exploits a parallel computation graph comprising multiple stages: MLP of width $h$, skew-symmetric scattering, matrix exponentiation, and a one-sided orthogonal decomposition. Table~\ref{tab:complexity} summarises the comparison with classical baselines. 
    \begin{table}[!h]
      \centering
      \caption{Work--span complexity~\cite{blelloch1996programming}.
        Full proofs are found in Supp.~\ref{sec:supp:complexity}.}
      \label{tab:complexity}
      \resizebox{.8\linewidth}{!}{%
      \begin{tabular}{@{}lcc@{}}
        \toprule
        \textbf{Method} & \textbf{Work} & \textbf{Span} \\
        \midrule
        Golub--Reinsch~\cite{golub1971singular}
          & $\mathcal{O}(mn^2)$ & $\mathcal{O}(m\log m)$ \\
        Truncated SVD (bidiag.)~\cite{hansen1987truncated}
          & $\mathcal{O}(mnK)$ & $\mathcal{O}(K\log n)$ \\
        Randomised SVD (rSVD)~\cite{halko2011finding}
          & $\mathcal{O}(mn\ell \!+\! n\ell^2)$ & $\mathcal{O}(\ell \log n)$ \\
        Lanczos (rank $K$)~\cite{golub2013matrix}
          & $\mathcal{O}(mnK)$  & $\mathcal{O}(K\log n)$ \\
        Subspace (block) iteration~\cite{musco2015randomized} 
          & $\mathcal{O}(mnK)$ & $\mathcal{O}(K\log n)$ \\
        \midrule
        \textbf{Lie-based SVD} by Cayley
          & $\mathcal{O}(m^3)$  & $\mathcal{O}(m\log m)$ \\
        \textbf{Lie-based SVD} by Rodrigues 
          & $\mathcal{O}(1)$ & $\mathcal{O}(1)$ \\
        \textbf{Lie-based SVD} by Taylor
          & $\mathcal{O}(m^3)$  & $\mathcal{O}(\log m)$ \\
        \bottomrule
      \end{tabular}}
    \end{table}
    While the Cayley-transform's LU solve dictates the span for general cases, employing a constant-degree Taylor expansion or the closed-form Rodrigues formula ($m\!=\!n\!=\!3$) reduces the span to $\mathcal{O}(\log m)$. 
    By default, a 9-term Taylor expansion is employed for general matrices, whereas the Rodrigues formula is used for $\mathbf{A}\in\mathbb{R}^{3\times3}$. %
    \subsection{Complexity Proofs for Lie-based SVD}
    \label{sec:supp:complexity}
    This section analyses the Lie-based SVD under the work–span model and provides formal derivations for all claims in Table~\ref{tab:complexity}.
\fi
Throughout, we write $n = \min(m,n)$, corresponding to the implicit rank used in the main
paper's complexity table. We adopt the EREW PRAM model~\cite{blelloch1996programming}: processors may not simultaneously read or write the same memory cell. All reductions
over $p$ elements use balanced binary fan-in trees.

\subsection*{Notation and Model Assumptions}

\begin{itemize}
  \item $\mathbf{A} \in \mathbb{R}^{m \times n}$, $m \ge n$ (enforced by transpose
    branch), $n = \min(m,n)$.
  \item MLP hidden width $h = \mathcal{O}(m)$.
  \item \emph{Work} $W$: total arithmetic operations.
  \item \emph{Span} $S$: depth of the critical path in the DAG of
    operations.
  \item Matrix multiply $\mathbb{R}^{a \times b} \cdot \mathbb{R}^{b
    \times c}$: work $\mathcal{O}(abc)$, span $\mathcal{O}(\log b)$
    (parallel dot products, each a reduction tree).
  \item By \textbf{Brent's theorem}~\cite{brent1974parallel}, wall-clock
    time on $P$ processors satisfies $T_P = \mathcal{O}(W/P + S)$.
\end{itemize}

\begin{quote}
\textbf{Batch dimension note.}
The work--span bounds in Table~\ref{tab:complexity} and 
Propositions~\ref{prop:preproc}--\ref{prop:total} are defined per single matrix instance $\mathbf{A} \in \mathbb{R}^{m \times n}$. 
Across a batch of size $B$, SIMT execution scales total work linearly as $\mathcal{O}(B \cdot W)$ while ideal span remains $S$ due to parallel evaluation across threads. 
Empirical scaling over batch dimensions is evaluated in Figure~\ref{fig:times:b}.
\end{quote}

\subsection*{Stage 1: Pre-processing MLP}
\label{sec:supp:stage1}

\begin{proposition}[Pre-processing work and span]
\label{prop:preproc}
A two-layer MLP $f:\mathbb{R}^{mn} \to \mathbb{R}^{mn}$ with hidden
width $h$ satisfies
\[
  W_{\mathrm{pre}} = \mathcal{O}(mnh), \qquad
  S_{\mathrm{pre}} = \mathcal{O}(\log(mn)).
\]
\end{proposition}

\begin{proof}
The MLP computes
$f(x) = W_2\,\sigma(W_1 x + b_1) + b_2$
with $W_1 \in \mathbb{R}^{h \times mn}$ and $W_2 \in \mathbb{R}^{mn
\times h}$.

\textbf{Layer 1.}  Each of the $h$ output neurons computes an inner
product over $mn$ inputs.  With unlimited processors, all $h$ inner
products run in parallel; each inner product is a reduction tree of
depth $\lceil \log_2(mn) \rceil$.  Hence
$W_1 = \mathcal{O}(h \cdot mn)$, $S_1 = \mathcal{O}(\log(mn))$.

\textbf{Activation.}  Elementwise $\sigma$ on $h$ values:
$W = \mathcal{O}(h)$, $S = \mathcal{O}(1)$.

\textbf{Layer 2.}  Each of the $mn$ outputs is an inner product over $h$
values: $W_2 = \mathcal{O}(mn \cdot h)$, $S_2 = \mathcal{O}(\log h)$.

Summing, and noting $\log h \le \log(mn)$:
$W_{\mathrm{pre}} = \mathcal{O}(mnh)$,
$S_{\mathrm{pre}} = \mathcal{O}(\log(mn))$.
\end{proof}

\subsection*{Stage 2: Reduction MLP}

\begin{proposition}[Reduction work and span]
\label{prop:reduce}
The reduction MLP $g:\mathbb{R}^{mn} \to \mathbb{R}^{\binom{m}{2}}$
with hidden width $h$ satisfies
\[
  W_{\mathrm{red}} = \mathcal{O}(m^2 h), \qquad 
  S_{\mathrm{red}} = \mathcal{O}(\log m).
\]
\end{proposition}

\begin{proof}
Identical argument to Proposition~\ref{prop:preproc}.  The output
dimension is $|\mathcal{T}| = m(m-1)/2$ with $\mathcal{O}(m^2)$. Since $m \ge n$,
$m^2 \ge mn$, so the output layer cost $\mathcal{O}(|\mathcal{T}| \cdot
h) = \mathcal{O}(m^2 h)$ exceeds the input
layer cost $\mathcal{O}(mnh)$ by $W_{\mathrm{red}} = \mathcal{O}(m^2 h)$.
\end{proof}

\subsection*{Stage 3: Skew-Symmetric Scatter}

\begin{proposition}[Fill skew-symmetric work and span]
\label{prop:skew}
Scattering $|\mathcal{T}|$ values into an $m \times m$ skew-symmetric
matrix satisfies
\[
  W_{\mathrm{skew}} = \mathcal{O}(m^2), \qquad
  S_{\mathrm{skew}} = \mathcal{O}(1).
\]
\end{proposition}

\begin{proof}
Each of the $|\mathcal{T}| = m(m-1)/2$ entries is written to two
locations $(i,j)$ and $(j,i)$ with sign flip; there are no dependencies
between writes (EREW is satisfied since all target cells are distinct).
Total writes: $m^2 - m = \mathcal{O}(m^2)$.  All writes execute in one
parallel step, so $S_{\mathrm{skew}} = \mathcal{O}(1)$.
\end{proof}

\subsection*{Stage 4: Matrix Exponential}

\subsubsection*{4a. Cayley Transform}

\begin{proposition}[Cayley transform work and span]
\label{prop:cayley}
The Cayley map $(I - S)^{-1}(I + S)$ for $S \in \mathbb{R}^{m \times m}$
skew-symmetric satisfies
\[
  W_{\mathrm{Cay}} = \mathcal{O}(m^3), \qquad
  S_{\mathrm{Cay}} = \mathcal{O}(m \log m).
\]
\end{proposition}

\begin{proof}
\textbf{LU decomposition of $(I - S)$.}
Gaussian elimination with partial pivoting proceeds through $m-1$
sequential elimination steps.  At step $k$, pivot search over $m-k$
rows costs span $\mathcal{O}(\log m)$, and row updates for the remaining
$m-k$ rows cost span $\mathcal{O}(\log m)$ (parallel scalar-multiply-add
across $m-k$ columns).  Across $m-1$ steps the span accumulates to
$\mathcal{O}(m \log m)$.  Total work for LU: $\mathcal{O}(m^3)$.

\textbf{Triangular solves.} Forward and back substitution each require
$m$ sequential steps with $\mathcal{O}(\log m)$ span per step:
$S = \mathcal{O}(m \log m)$, $W = \mathcal{O}(m^2)$.

\textbf{Matrix multiply $(I-S)^{-1}(I+S)$.}
$\mathbb{R}^{m\times m}\cdot\mathbb{R}^{m\times m}$:
$W = \mathcal{O}(m^3)$, $S = \mathcal{O}(\log m)$.

Combining (LU dominates): $W_{\mathrm{Cay}} = \mathcal{O}(m^3)$,
$S_{\mathrm{Cay}} = \mathcal{O}(m \log m)$.
\end{proof}

\subsubsection*{4b. Rodrigues Formula ($m = n = 3$)}

\begin{proposition}[Rodrigues work and span]
\label{prop:rodrigues}
For $m = n = 3$, the Rodrigues rotation formula satisfies
\[
  W_{\mathrm{Rod}} = \mathcal{O}(1), \qquad S_{\mathrm{Rod}} = \mathcal{O}(1).
\]
\end{proposition}

\begin{proof}
With $m = 3$, a skew-symmetric $S$ has exactly 3 free parameters
$(\theta_1, \theta_2, \theta_3)$.  The Rodrigues formula
\[
  X = I + \frac{\sin\theta}{\theta}\,S +
      \frac{1 - \cos\theta}{\theta^2}\,S^2,
  \quad \theta = \|\mathbf{v}\|_2,\;
  \mathbf{v} = (\theta_1, \theta_2, \theta_3)^\top,
\]
involves a fixed number of scalar operations ($\|\cdot\|$, $\sin$,
$\cos$, divisions) and two $3 \times 3$ matrix-scalar products, all
independent of any variable-size dimension.  Hence both work and span
are $\mathcal{O}(1)$.
\end{proof}

\subsubsection*{4c. Taylor Series (degree $p$)}

\begin{proposition}[Taylor exponential work and span]
\label{prop:taylor}
A degree-$p$ Horner-form Taylor approximation $e^S \approx \sum_{k=0}^p
S^k/k!$ for $S \in \mathbb{R}^{m \times m}$ satisfies
\[
  W_{\mathrm{Tay}} = \mathcal{O}(pm^3), \qquad
  S_{\mathrm{Tay}} = \mathcal{O}(p \log m).
\]
For constant $p = \mathcal{O}(1)$ this reduces to
$W = \mathcal{O}(m^3)$, $S = \mathcal{O}(\log m)$.
\end{proposition}

\begin{proof}
Horner evaluation $(((\frac{1}{p!}S + \frac{1}{(p-1)!}I)S + \cdots)S
+ I)$ requires $p$ sequential matrix multiplications, each costing
$W = \mathcal{O}(m^3)$ and $S = \mathcal{O}(\log m)$. Sequentiality of
the $p$ steps gives total span $p \cdot \mathcal{O}(\log m) =
\mathcal{O}(p \log m)$.
\end{proof}

\subsection*{Stage 5: One-Sided Decomposition}

\begin{proposition}[Decomposition work and span]
\label{prop:decomp}
Given $\mathbf{A} \in \mathbb{R}^{m \times n}$ and $\mathbf{X} \in \mathbb{R}^{m \times
m}$ orthogonal, computing $(\mathbf{X}, \mathbf{D}, \mathbf{Y})$
satisfies %
\[
  W_{\mathrm{dec}} = \mathcal{O}(m^2n), \qquad
  S_{\mathrm{dec}} = \mathcal{O}(\log m).
\]
\end{proposition}

\begin{proof}
The routine performs the following operations in order.

\textbf{(i) $\mathbf{A}\mathbf{A}^\top \in \mathbb{R}^{m \times m}$:}
$W = \mathcal{O}(m^2 n)$, $S = \mathcal{O}(\log n)$.

\textbf{(ii) $(\mathbf{A}\mathbf{A}^\top)\mathbf{X}_n \in \mathbb{R}^{m \times n}$:}
$W = \mathcal{O}(m^2 n)$, $S = \mathcal{O}(\log m)$.

\textbf{(iii) Diagonal extraction:}
Diagonal extraction over $n$ columns, $m$ rows each: $W = \mathcal{O}(mn)$, $S = \mathcal{O}(\log m)$; $n$ square roots.

\textbf{(iv) Right factor $\mathbf{Y} = \mathbf{A}^\top\mathbf{X}_n\mathbf{D}^\dagger \in \mathbb{R}^{n \times n}$:} $\mathbf{A}^\top\mathbf{X}_n \rightarrow n\times n$: $W=\mathcal{O}(mn^2)$, $S=\mathcal{O}(\log m)$; then $\mathbf{D}^\dagger$ scaling: $W=\mathcal{O}(n^2)$, $S=\mathcal{O}(1)$.

Since $m\ge n$, $m^2n\ge mn^2$, so the dominant work term is $\mathcal{O}(m^2n)$, and the dominant span is $\mathcal{O}(\log m)$.
\end{proof}

\subsection*{End-to-End Theorem}

\begin{proposition}[End-to-end complexity]
\label{prop:total}
Under EREW PRAM with $h = \mathcal{O}(m)$ and $m \ge n$, the full
neural Lie-SVD forward pass satisfies:
\begin{align*}
  &\textit{Cayley:} \,\,
    W = \mathcal{O}(m^3),\,\, S = \mathcal{O}(m \log m). \\
  &\textit{Rodrigues, } m\!=\!n\!=\!3\textit{:} \,\,
  W = \mathcal{O}(1),\,\, S = \mathcal{O}(1). \\
  &\textit{Taylor, degree } p\!=\!\mathcal{O}(1)\textit{ (default):} \,\,
    W = \mathcal{O}(m^3),\,\, S = \mathcal{O}(\log m).
\end{align*}
\end{proposition}

\begin{proof}
Sum Propositions~\ref{prop:preproc}--\ref{prop:decomp} for each variant,
noting that stages execute sequentially so spans add while works add.

\textbf{Cayley.}
$W = \mathcal{O}(mn^2 + m^2n + m^2 + m^3 + m^2n) = \mathcal{O}(m^3)$. Since $h = \mathcal{O}(m)$, the preprocessing MLP costs $mnh = \mathcal{O}(mn^2)$; the reduction MLP costs $\mathcal{O}(m^2 h)=\mathcal{O}(m^2 n)$; Stage 5 costs $\mathcal{O}(m^2 n)$. Since $m\ge n$, $m^3\ge m^2n\ge mn^2$, so the Cayley exponential (Stage 4) dominates: $W=\mathcal{O}(m^3)$. Span: $\mathcal{O}(\log(mn)) + \mathcal{O}(\log m) + \mathcal{O}(1) + \mathcal{O}(m\log m) + \mathcal{O}(\log m) = \mathcal{O}(m \log m)$ (LU dominates).

\textbf{Rodrigues.}
With dimensions fixed at $m = n = 3$, all five pipeline stages consume constant work and execution depth, yielding $W = \mathcal{O}(1)$ and $S = \mathcal{O}(1)$.

\textbf{Taylor ($p = \mathcal{O}(1)$).}
As in the Cayley case, $m^3$ dominates $\mathcal{O}(m^2n)$ from Stage 5.
Stage 4 span becomes $\mathcal{O}(\log m)$; Stage 5 span $\mathcal{O}(\log m)$
also contributes but does not exceed it, so $S = \mathcal{O}(\log m)$.
Work as in the Cayley case: $\mathcal{O}(m^3)$.
\end{proof}

\begin{corollary}[GPU speedup via Brent's theorem]
\label{cor:brent}
On a GPU with $P$ SIMD lanes, the wall-clock time satisfies
\[
  T_P = \mathcal{O}\!\left(\frac{W}{P} + S_{\mathrm{total}}\right).
\]
The speedup relative to single-threaded Golub--Reinsch
($T_1^{\mathrm{GR}} = \mathcal{O}(mn^2)$, $S^{\mathrm{GR}} =
\mathcal{O}(m \log m)$) is
\[
  \frac{T_1^{\mathrm{GR}}}{T_P} = \mathcal{O}\!\left(\frac{mn^2}{W/P+S_{\mathrm{total}}}\right)
\]
which is meaningful when $P \gg r$, i.e.\ the GPU is sufficiently
wide relative to the matrix rank.
\end{corollary}

\textbf{Memory-access patterns.}
The span bounds above assume unit-cost memory access.  In practice,
GPU coalesced access requires column-major layout for the matrix multiply
in Stage 5 and row-major for the MLP weight matrices.  Cache-miss
penalties can inflate effective span by a constant factor that depends on
matrix tile size relative to shared-memory capacity; this does not affect
asymptotic span but is important for constant-factor comparisons.

\section{Attempts that did not work}
\label{supp:sec:attempts}

\subsection{Two-sided orthogonality constraint}

The proposed orthogonality constraint can be simultaneously applied to $\mathbf{X}$ and $\mathbf{Y}$. This \textit{two-sided} constraint requires matching orthogonal singular vectors in $\mathbf{X}$ and $\mathbf{Y}$ to satisfy
$\mathbf{A}=\mathbf{X}\mathbf{D}\mathbf{Y}^\top$. In practice, the neural architectures used in this study fail to reconstruct valid combinations when enforcing both $\mathbf{X}\in\textrm{SO}(m)$ and $\mathbf{Y}\in\textrm{SO}(n)$ such that predicted components result in $\mathbf{A}\neq\mathbf{X}\mathbf{D}\mathbf{Y}^\top$. Consequently, the inability to reconstruct $\mathbf{A}$ from $\mathbf{X}\mathbf{D}\mathbf{Y}^\top$ does not qualify to be a valid decomposition technique. However, it is worth mentioning that this \textit{two-sided} constraint achieved SSIM scores similar to the \textit{one-sided} equivalent when trained using the $\mathcal{L}_{\mathrm{E2E}}$ loss from Eq.~(\ref{eq:loss:end2end}).

\subsection{Energy-based optimisation}
\label{supp:subsec:optimization}

Building on the notation of Eqs.~(\ref{eq:loss:sigma}), (\ref{eq:loss:off}), and (\ref{eq:loss:ort}) of the main paper, one can construct the following overall objective:
\begin{equation}
	\mathbf{s}^{\star}=\underset{\mathbf{s} \in \mathbb{R}^{L}}{\arg\min} \; \left\{ 
	\lambda_1\mathcal{L}_{\sigma} +
	\lambda_2\mathcal{L}_{\text{off}} +
	\lambda_3\mathcal{L}_{\text{ort}}
	\right\},
	\label{eq:optimization}
\end{equation}
for energy-based optimisation. In this work, the optimisation objective is solved using the limited-memory Broyden--Fletcher--Goldfarb--Shanno algorithm (LBFGS) as an efficient Hessian-based optimiser.
From the optimal estimate $\mathbf{s}^{\star}$, the components $\mathbf{X}^\star$, $\mathbf{D}^\star$, and $\mathbf{Y}^\star$ are obtained analogous to Section~\ref{sec:method} of the main paper. \par
Table~\ref{tab:energy} supplements the MMP benchmark results in Table~\ref{tab:metrics:npp} of the main paper. The energy-based solutions are close to the corresponding initialization while incurring additional computational cost. 

\begin{table*}[!ht]
    \centering 
    \caption{Energy-based optimisation results complementing the NPP experiment from Table~\ref{tab:metrics:npp}.}\label{tab:energy}
    \resizebox{.85\textwidth}{!}{%
        \begin{tabular}{l c 
        S[table-format=2.2(2),detect-all] 
        S[table-format=2.2(2),detect-all] 
        S[table-format=2.2(2),detect-all] 
        S[table-format=2.2(2),detect-all] 
        S[table-format=2.2(2),detect-all] 
        }
        \toprule
        Method & RNN & $\overline{\mathrm{SSIM}(\mathbf{M})}\uparrow$ & $\overline{\|\mathbf{X}\|_F^{\text{ort}}}\downarrow$ & $\overline{\|\mathbf{D}\|_F^{\text{rel}}}\downarrow$ & $\overline{\|\mathbf{Y}\|_F^{\text{ort}}}\downarrow$ & $\text{Time}~\text{[ms]}\downarrow$ \\
        \midrule
        Analytic (Baseline) & \ding{55} & 0.66(0.08) & 0.00(0.00) & 0.09(0.04) & 0.46(0.08) & 8.48(0.02) \\
        LieNN  (Baseline) & \ding{51} & 0.94(0.03) & 0.00(0.00) & 0.01(0.00) &  0.10(0.03) & 14.19(0.02) \\
        \midrule
        Energy (Analytic init) & \ding{55} & 0.69(0.08) & 0.00(0.00) & 0.08(0.03) & 0.44(0.08) & 28.13(0.04) \\
        Hybrid (LieNN init) & \ding{51} & 0.94(0.03) & 0.00(0.00) & 0.01(0.00) & 0.10(0.03) & 33.69(0.08) \\
        \bottomrule
        \end{tabular}
    }
\end{table*}

\subsection{Hybrid learning and energy optimisation}

Using the energy-based optimisation as an extension for the neural decompositions in Section~\ref{sec:method} of the main paper did not improve on the results but added a large computational overhead.

\subsection{Learned recurrent optimiser}

While the RNN is used as $\mathbf{A}^{(t+1)}=\textrm{RNN}\left(\mathbf{A}^{(t)}\right)$  with step $t$ in the presented results, an alternative approach implies computing residuals $\mathbf{r}$ from:
\begin{align*}
	\mathbf{r} = \mathbf{d}_{\mathbf{X}} - \mathbf{d}_{\mathbf{Y}}
\end{align*}
with diagonal vectors $\mathbf{d}_{\mathbf{X}}$ and $\mathbf{d}_{\mathbf{Y}}$ from Eq.~(\ref{eq:diag_closed}) in the main paper for $\mathbf{X}$ and $\mathbf{Y}$, respectively. The intuition is to then let networks be aware of singular value distances by:
\begin{align*}
	\mathbf{D}^{(t+1)}_{\mathbf{X}}=\mathbf{D}^{(t)}_{\mathbf{X}} + \textrm{RNN}_{\mathbf{X}}(\mathbf{r}^{(t)}) \, , \\
	\mathbf{D}^{(t+1)}_{\mathbf{Y}}=\mathbf{D}^{(t)}_{\mathbf{Y}} + \textrm{RNN}_{\mathbf{Y}}(\mathbf{r}^{(t)}) \, ,
\end{align*}
to achieve a matching set of left and right singular vectors. \par
The results of such learned optimisers remain below 0.5 SSIM indicating that the network architectures struggle to find underlying data patterns using this scheme.

\section{Additional Experimental Results}

This appendix provides extended results that exceed the space constraints of the main paper.

\subsection{Numerical error growth}

We analyze how the metrics in Section~\ref{ssec:metrics} of the main paper evolve with the matrix dimension $m$.
Since publicly available datasets rarely contain matrices sampled at uniform size intervals, we generate synthetic data with known ground-truth factors, as detailed below. Figure~\ref{fig:error} shows how error metrics evolve with the square matrix dimension $m$.

\begin{figure*}[!h]
    \begin{minipage}[t]{0.05\linewidth}
        \raisebox{1.0\height}{\rotatebox{90}{Error~[a.u.]}}
    \end{minipage}%
    \begin{minipage}[t]{0.94\linewidth}
        \centering 
        \includegraphics[width=\linewidth]{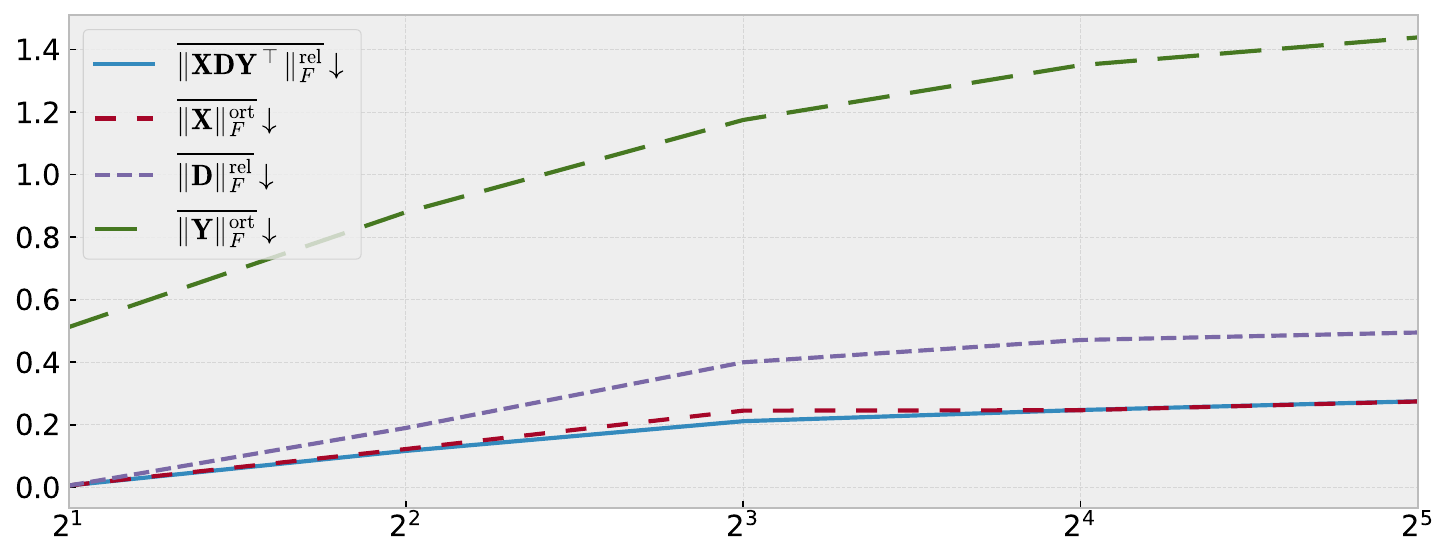}
        Matrix dimension $m$ [\#]
    \end{minipage}
    \hspace{-6cm}
    \caption{\textbf{Error growth for varying dimension $m$ in square matrices.} The curves depict the metrics from Section~\ref{ssec:metrics} of the main paper, representing the orthogonality errors $\overline{\|\mathbf{X}\|_F^{\text{ort}}}$ and $\overline{\|\mathbf{Y}\|_F^{\text{ort}}}$ as well as deviations for singular values in $\overline{\|\mathbf{D}\|_F^{\text{rel}}}$ and a matrix reconstruction error $\overline{\|\mathbf{X}\mathbf{D}\mathbf{Y}^{\top}\|_F^{\text{rel}}}$.\label{fig:error}}
\end{figure*}

\ifJournalOption
\else
    \subsection{Scalability across matrix and batch sizes} %
    To complement the per-benchmark runtimes reported in the main paper, we evaluate how the best-performing approaches scale with problem size. Figure~\ref{fig:times} plots measured execution times as a function of matrix dimension and batch size. %
    The matrix-dimension sweep (Fig.~\ref{fig:times:a}) reflects the increasing number of temporal frames utilised in ULM clutter removal, where deep temporal stacks are often preferred~\cite{demene2015spatiotemporal,baranger2018adaptive,heiles2022pala}. Conversely, the batch-size sweep (Fig.~\ref{fig:times:b}) corresponds to the spatial pixel count in a Mueller matrix image, where each pixel requires an independent 3×3 decomposition. \par
    \begin{figure*}[!hb]
        \centering
        \begin{minipage}[t]{.48\linewidth}
            \begin{overpic}[width=1\linewidth]{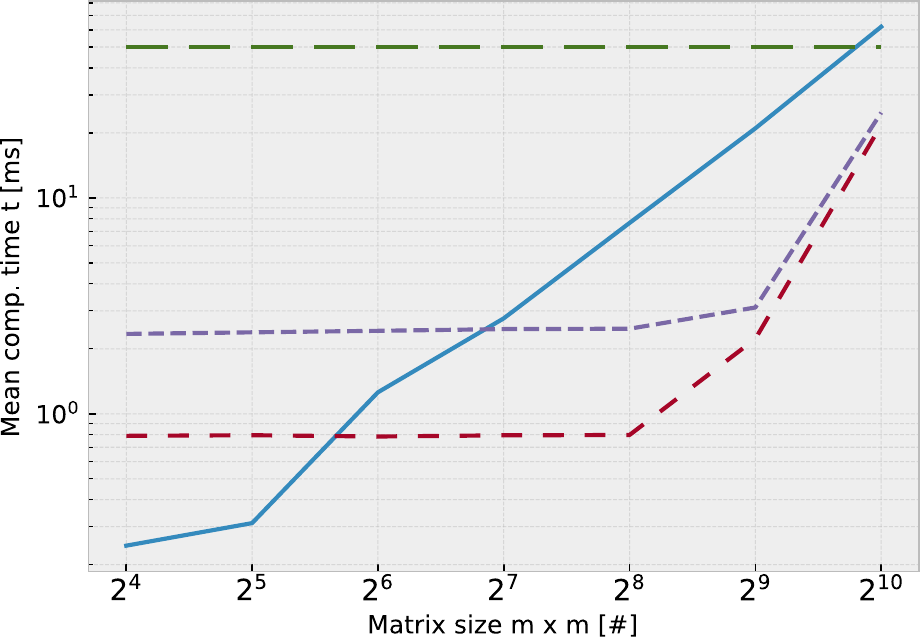}
              \put(15,40){\includegraphics[width=0.4\linewidth]{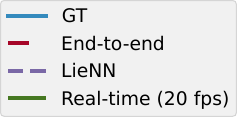}}
            \end{overpic}
            \subcaption{Varying matrix size $m=n$ at constant batch dimension 1\label{fig:times:a}}
        \end{minipage}
        \hfill
        \begin{minipage}[t]{.48\linewidth}
            \begin{overpic}[width=1\linewidth]{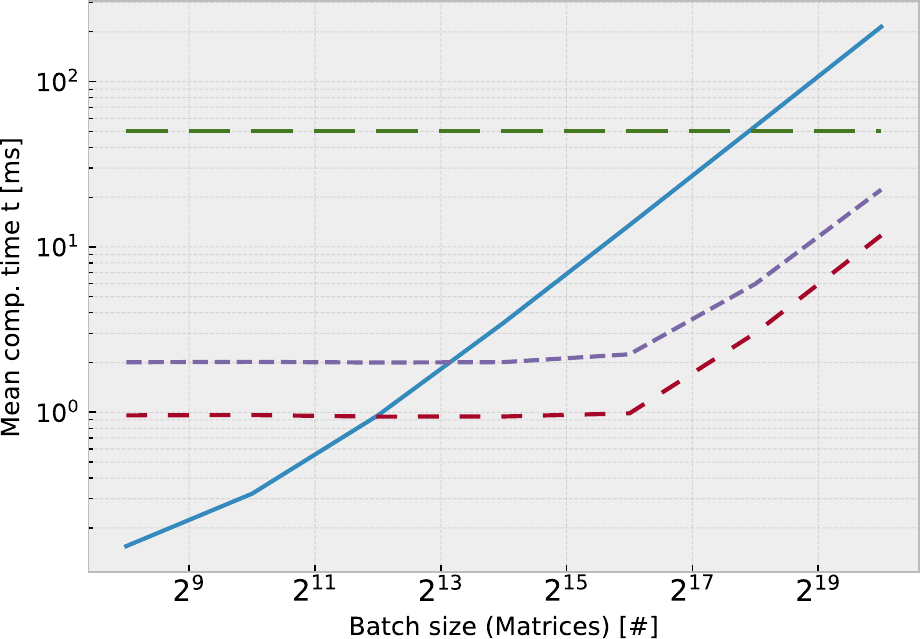}
              \put(15,45){\includegraphics[width=0.4\linewidth]{figs/compute_times_plot_bs_3_log_legend.pdf}}
            \end{overpic}
            \subcaption{Varying batch size at constant matrix dimension $m=n=3$\label{fig:times:b}}
        \end{minipage}
        \caption{\textbf{Computation time plots across matrix and batch sizes.} Timings were measured on an Nvidia H100 GPU to accommodate large memory demands and averaged over 1000 iterations after 100 preceding warm-ups each. Runtime increases sharply with data size, amplifying the gap between the classical SVD and the proposed methods. The plots use logarithmic axes (base-10 on y, base-2 on x), visually compressing exponential growth in runtime. \label{fig:times}}
    \end{figure*}
    The plots use logarithmic axes (base-10 on the vertical and base-2 on the horizontal axis), which compress the rapid growth of absolute runtimes and let the curves appear deceptively linear. Up to moderate sizes ($\approx 2^6$ in matrix size; $\approx2^{13}$ in batch size), presented methods behave similarly while the cuSOLVER-based SVD benefits from the Jacobi routine, which is optimal for small matrices ($m,n \le 32$). Beyond these ranges, the absolute gaps widen substantially although the log scaling masks their steep divergence. This underscores the practical advantage of the proposed approach as data scales toward high-throughput regimes.
\fi

\subsection{Accuracy analysis for randomised SVD and Taylor-maps}
\label{supp:sec:rsvd}

\paragraph{Why rSVD cannot accelerate small full-rank matrices} With $q=\min(m,n)$, the rSVD sketch spans the full space, so the method cannot save work over the exact SVD: after its power iterations, it computes an exact SVD of a matrix of the same size. On one RTX 5090, \texttt{torch.svd\_lowrank} ($q=3$, nine power iterations, as in Tab.~\ref{tab:metrics:npp}) takes 70~s for the 200{,}208 matrices of an NPP frame, almost entirely in its 19 QR factorisations (3.7~s each): the Householder factorisation itself runs batched (0.04~ms), but forming the orthogonal factor issues one cuSOLVER call per matrix (18~$\mu$s each). Randomised SVD is thus unsuited to large batches of small full-rank matrices for two reasons, the per-matrix execution of library QR routines and the absence of any low-rank saving, whereas LieNN decomposes them with batched matrix products only.

\paragraph{Learned MMP factors} To analyse the learned MMP factors, whose checkpoints are not available, LieNN+RNN was retrained with the published configuration (40 epochs) on the RTX 5090. It reaches SSIM 0.93 on NPP and 0.82 on MUC; on NPP, its left factor deviates from the exact singular vectors by median angles of 6.5$^\circ$, 8.0$^\circ$ and 1.3$^\circ$, and the resulting retarder by a median geodesic error of 1.1$^\circ$ (6.7$^\circ$ on MUC).

\paragraph{Taylor-map accuracy} MMP uses Rodrigues' closed form, which is exactly orthogonal; MNIST and ULM use the 9-term Taylor series without scaling and squaring. Table~\ref{tab:expcheck} re-evaluates the MNIST models with the Taylor map and with an exact exponential (\texttt{torch.linalg.matrix\_exp}) at unchanged weights. For the analytic variant, $\lVert\mathbf{S}\rVert_2\approx14$ lies far outside the accuracy range of the truncated series, which explains its divergence in Tab.~\ref{tab:metrics:mnist}; LieNN predicts generators with $\lVert\mathbf{S}\rVert_2\approx4.3$, and replacing the exponential leaves its PSNR unchanged.
\begin{table}[!ht]
    \centering
    \caption{MNIST test set ($280\times280$): 9-term Taylor vs.\ exact exponential at unchanged weights (re-evaluated on an RTX 5090; deviations from Tab.~\ref{tab:metrics:mnist} $\le$0.03\,dB).\label{tab:expcheck}}
    \resizebox{.6\columnwidth}{!}{%
    \begin{tabular}{llcccc}
        \toprule
        Model & $\exp(\cdot)$ & $k$ & PSNR [dB] & $\overline{\|\mathbf{X}\|_F^{\text{ort}}}$ & $\overline{\lVert\mathbf{S}\rVert_2}$ \\
        \midrule
        LieNN & Taylor-9 & 96 & 17.78 & 0.21 & 4.25 \\
        LieNN & exact & 98 & 17.79 & $1.0\cdot10^{-5}$ & 4.25 \\
        LieNN+RNN & Taylor-9 & 95 & 17.93 & 0.13 & 4.25 \\
        LieNN+RNN & exact & 90 & 17.91 & $1.1\cdot10^{-5}$ & 4.25 \\
        Analytic & Taylor-9 & 280 & 2.97 & $3.6\cdot10^{8}$ & 13.98 \\
        Analytic & exact & 278 & 15.40 & $1.8\cdot10^{-5}$ & 13.98 \\
        \bottomrule
    \end{tabular}}
\end{table}

\subsection{Additional MNIST test results}
%
Figure~\ref{fig:mnist:sweep} shows the PSNR of each method as a function of truncation rank $k$ on the MNIST patchwork test set. All rank-truncated methods are evaluated at every integer rank from 1 to 280. The per-method optima reported in Table~\ref{tab:metrics:mnist} of the main paper correspond to the peak of each curve on the validation split. LieNN maintains a consistent advantage over classical SVD across a broad range of $k$ values, indicating that its benefit is not sensitive to precise rank tuning. %
\begin{figure*}[!h]
    \centering
    \includegraphics[width=\linewidth]{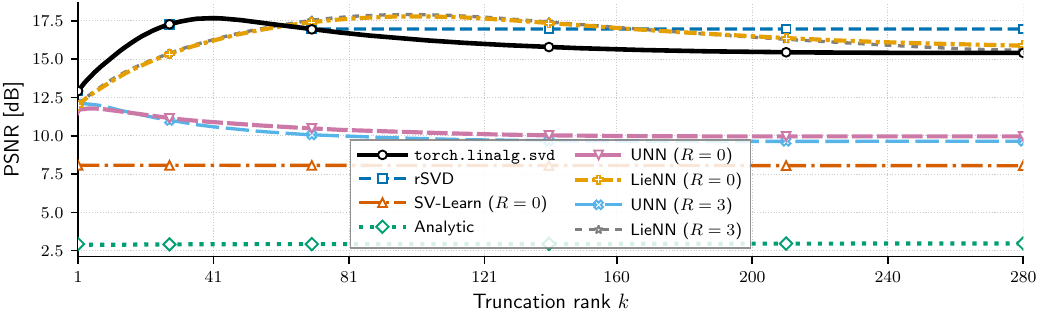}
    \caption{\textbf{MNIST patchwork PSNR rank sweep} ($\mathbf{A}_i\in\mathbb{R}^{280\times280}$). PSNR relative to the clean ground truth as a function of truncation rank $k$ for all rank-truncated methods. $R$ denotes the RNN iteration number, while $R=0$ indicates that the RNN is absent. Curves are evaluated at every integer rank from 1 to 280 on the test split; per-method optima are selected on the validation split.\label{fig:mnist:sweep}}
\end{figure*}

\subsection{Additional polarimetric test results}

\begin{figure*}[!ht]
	\centering
	\begin{minipage}[c]{.95\linewidth}
		\includegraphics[width=\linewidth]{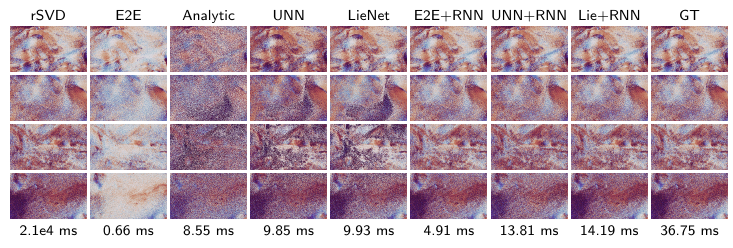}
	\end{minipage}
	\hfill
	\begin{minipage}[c]{.04\linewidth}
		\includegraphics[width=\linewidth]{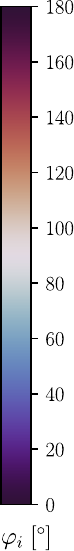}
	\end{minipage}
	\caption{\textbf{MMP benchmark comparison (pt. 2).} The images show the remaining NPP test samples of the per-pixel azimuth $\varphi_i$ decomposed by the various models. \label{fig:azimuth:remaining}}
\end{figure*}

Figure~\ref{fig:azimuth:remaining} provides the remaining azimuth images from the NPP test set, complementing those shown in Fig.~\ref{fig:azimuth} of the main paper. The visual trends mirror the quantitative findings reported earlier.

\subsection{Additional ULM test results}
\label{supp:sec:ulm}


\paragraph{Runtime scaling} Table~\ref{tab:ulmscaling} reports the per-stack runtime (measured without PyTorch's deterministic mode) versus the number $S$ of jointly filtered stacks and the stack length $n$. LieNN+RNN consists of batched matrix products only, so its per-stack time falls from 2.91~ms for a single stack to 0.91~ms for $S\geq4$, whereas the per-stack times of the exact solvers and rSVD, which runs its QR and SVD steps matrix by matrix, do not decrease with $S$. As LieNN's cost grows with $n^3$ and the rSVD sketch with $n^2q$, LieNN leads rSVD by 4.5$\times$ at $n=250$ and 1.5$\times$ at $n=500$ but is 1.9$\times$ slower at $n=800$.
\begin{table}[!ht]
    \centering
    \caption{ULM filter runtime per stack [ms] on one RTX 5090 versus jointly filtered stacks $S$ ($n=250$) and stack length $n$ ($S=16$; Exact SVD and \texttt{eigh} at $S=1$, as their per-stack time does not depend on $S$).\label{tab:ulmscaling}}
    \resizebox{.5\columnwidth}{!}{%
    \begin{tabular}{lcccccc}
        \toprule
        Method & $S=1$ & $S=4$ & $S=15$ & $S=96$ \\
        \midrule
        Exact SVD & 60.95 & -- & 60.59 & -- \\
        \texttt{eigh} & 32.79 & -- & 31.58 & -- \\
        rSVD & 5.63 & 4.19 & 4.12 & 4.10 \\
        LieNN+RNN & 2.91 & 0.92 & 0.91 & 0.96 \\
        \bottomrule
    \end{tabular}}
\end{table}

\paragraph{Clutter subspace} The filter removes $\mathrm{span}(\mathbf{X}_{1:4})$; Table~\ref{tab:retrain:ulm} reports its principal angles to the exact clutter subspace $\mathrm{span}(\mathbf{U}_{1:4})$. LieNN+RNN captures the dominant tissue component $\mathbf{u}_1$ (98.9\% of the stack energy) to within 0.4$^\circ$, whereas the weaker clutter components deviate by up to 47$^\circ$; a subspace-aware training loss is left for future work.
\begin{table}[!ht]
    \centering
    \caption{ULM clutter subspace on the test stacks (15 stacks $\times$ 11 blocks): principal angles $\theta$ between $\mathrm{span}(\mathbf{X}_{1:4})$ and $\mathrm{span}(\mathbf{U}_{1:4})$ and angle between $\mathbf{u}_1$ and $\mathrm{span}(\mathbf{X}_{1:4})$. LieNN+RNN: mean (std) over three seeds.\label{tab:retrain:ulm}}
    \resizebox{.7\columnwidth}{!}{%
    \begin{tabular}{lcccc}
        \toprule
        Method & SSIM & $\theta_{\min}$ [$^\circ$] & $\theta_{\max}$ [$^\circ$] & $\angle(\mathbf{u}_1,\mathbf{X}_{1:4})$ [$^\circ$] \\
        \midrule
        \texttt{eigh} & 1.00 & 0.00 & 0.09 & 0.00 \\
        rSVD & 1.00 & 0.00 & 1.19 & 0.00 \\
        LieNN+RNN & 0.95(0.00) & 0.34(0.05) & 47.16(0.17) & 0.38(0.05) \\
        \bottomrule
    \end{tabular}}
\end{table}

\vfill

\end{document}